\documentclass[journal,twocolumn]{IEEEtran}

\ifCLASSINFOpdf
\else
\fi

\usepackage{amsmath,graphicx}
\usepackage{amssymb}
\usepackage{amsmath}
\usepackage{cite}
\usepackage{amsthm}
\usepackage{float}
\usepackage{bm}
\usepackage{color}
\usepackage{hyperref}

\def\gbf{{\mathbf{g}}}
\def\ubf{{\mathbf{u}}}
\def\xbf{{\mathbf{x}}}
\def\ybf{{\mathbf{y}}}

\def\ebf{{\mathbf{e}}}
\def\ybf{{\mathbf{y}}}
\def\fbf{{\mathbf{f}}}
\def\sbf{{\mathbf{s}}}
\def\rbf{{\mathbf{r}}}
\def\vbf{{\mathbf{v}}}
\def\abf{{\mathbf{a}}}
\def\wbf{{\mathbf{w}}}

\def\Hbf{{\mathbf{H}}}
\def\Gbf{{\mathbf{G}}}
\def\Dbf{{\mathbf{D}}}

\def\Abf{{\mathbf{A}}}
\def\Ibf{{\mathbf{I}}}

\def\Rcal{\mathcal{R}}
\def\Dcal{\mathcal{D}}

\def\Zcal{\mathcal{Z}}

\def\H{\text{H}}

\def\R{\mathbb{R}}

\def\prox{\mathrm{prox}}
\def\grad{\nabla}
\def\divergence{\nabla\cdot}

\newcommand{\diag}[1]{\text{diag}( #1 )}
\newcommand*{\tensor}[1]{\overline{\overline{#1}}}
\newcommand*{\conj}[1]{\overline{#1}}

\theoremstyle{definition}
\newtheorem{proposition}{Proposition}

\graphicspath{{figures/}}

\hypersetup{
    colorlinks=true,       
    linkcolor=blue,          
    citecolor=green,        
    urlcolor=blue           
}

\title{Accelerated Image Reconstruction for Nonlinear Diffractive Imaging }

\author{Yanting~Ma,~\IEEEmembership{Student~Member,~IEEE},
Hassan~Mansour,~\IEEEmembership{Senior~Member,~IEEE},
Dehong~Liu,~\IEEEmembership{Senior Member,~IEEE}, Petros~T.~Boufounos,~\IEEEmembership{Senior~Member,~IEEE},
and Ulugbek~S.~Kamilov,~\IEEEmembership{Member,~IEEE}
\thanks{Y.~Ma is with the 
Department of Electrical and Computer Engineering, North Carolina State University, Raleigh, NC 27606. This work was completed while Y.~Ma was with MERL.}
\thanks{H.~Mansour, D.~Liu, and P.~T.~Boufounos 
are with Mitsubishi Electric Research Laboratories (MERL), 201 Broadway, Cambridge,
MA 02139.}
\thanks{ U.~S.~Kamilov (email: kamilov@wustl.edu) is with Computational Imaging Group (CIG), Washington University in St.~Louis, St.~Louis, MO 63130.}}

\begin{document}

\IEEEoverridecommandlockouts
\maketitle
\thispagestyle{empty}


\begin{abstract}
The problem of reconstructing an object from the measurements of the light it scatters is common in numerous imaging applications. While the most popular formulations of the problem are based on linearizing the object-light relationship, there is an increased interest in considering nonlinear formulations that can account for multiple light scattering. In this paper, we propose an image reconstruction method, called CISOR, for nonlinear diffractive imaging, based on a nonconvex optimization formulation with total variation (TV) regularization. The nonconvex solver used in CISOR is our new variant of fast iterative shrinkage/thresholding algorithm (FISTA). We provide fast and memory-efficient implementation of the new FISTA variant and prove that it reliably converges for our nonconvex optimization problem. In addition, we systematically compare our method with other state-of-the-art methods on simulated as well as experimentally measured data in both 2D and 3D settings.
\end{abstract}





\section{Introduction}
\label{Sec:Intro}


Estimation of the spatial permittivity distribution of an object from the scattered wave measurements is ubiquitous in numerous applications. Conventional methods usually rely on linearizing the relationship between the permittivity and the wave measurements. For example, the first Born approximation \cite{Born.Wolf2003} and the Rytov approximation \cite{Devaney1981} are linearization techniques commonly adopted in diffraction tomography \cite{Bronstein.etal2002,Lim.etal2015,Sung.Dasari2011,Lauer2002,Sung.etal2009,Kim.etal2014a}. Other imaging systems that are based on a linear forward model include optical projection tomography (OPT), optical coherence tomography (OCT), digital holography, and subsurface radar \cite{Sharpe.etal2002,Choi.etal2007a,Ralston.etal2006,Davis.etal2007,Brady.etal2009,Tian.etal2010,Chen.etal2015a,Jol2009,Leigsnering.etal2014a,Liu.etal2016a}.
One attractive aspect of linear methods is that the inverse problem can be formulated as a convex optimization problem and solved by various efficient convex solvers \cite{Boyd.Vandenberghe2004,Nocedal.Wright2006,Bioucas-Dias.Figueiredo2007,Beck.Teboulle2009}.
However, linear models are highly inaccurate when the physical size of the object is large compared to the wavelength of the incident wave or the permittivity contrast of the object compared to the background is high \cite{Chen.Stamnes1998}. 
Therefore, in order to be able to image strongly scattering objects such as human tissue \cite{Ntziachristos2010}, nonlinear formulations that can model multiple scattering need to be considered. The challenge is then to develop fast, memory-efficient, and reliable inverse methods that can account for the nonlinearity. Note that the nonlinearity in this work refers to the fundamental relationship between the scattered wave and the permittivity contrast, rather than that introduced by the limitations of sensing systems such as missing phase.

\begin{figure}[t]
\begin{center}
\includegraphics[width=0.49\textwidth]{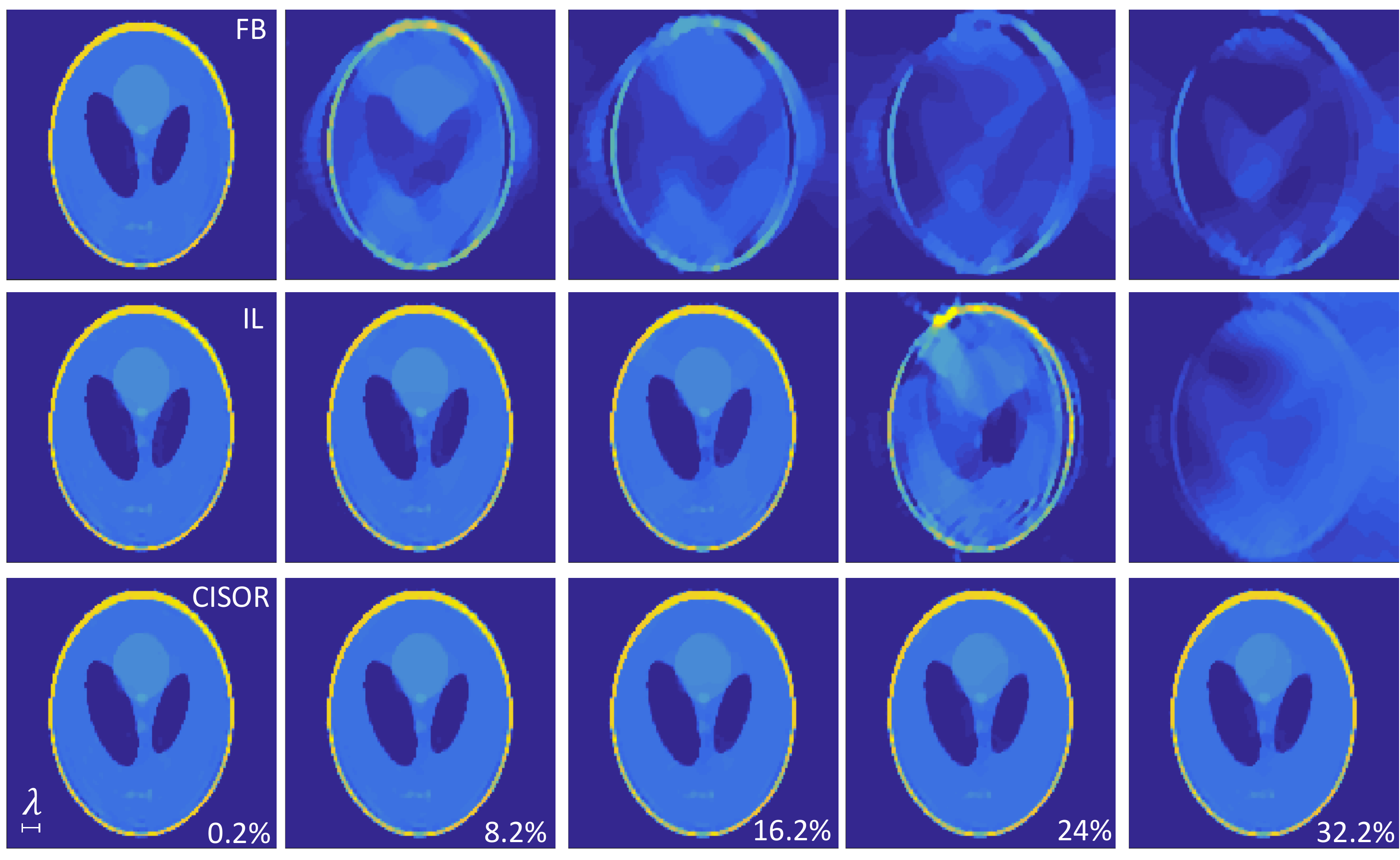}
\end{center}
\caption{Comparison between linear method with first Born approximation (FB, first row), iterative linearization method (IL, second row), and our proposed nonlinear method (CISOR, third row). Each column represents one contrast value as indicated at the bottom of the images on the third row. CISOR is stable for all tested contrast values, whereas FB and IL fail for large contrast.}
\label{fig:SheppLogan2D}
\vspace{-3mm}
\end{figure}


A standard way for solving inverse problems is via optimization, where a sequence of estimates is generated by iteratively minimizing a cost function. For ill-posed problems, a cost function usually consists of a quadratic data-fidelity term and a regularization term, which incorporates prior information such as transform-domain sparsity to mitigate the ill-posedness.  Total variation (TV) \cite{Rudin.etal1992} is one of the most commonly used regularizers in image processing, as it captures the property of piece-wise smoothness in natural objects. The challenge of such formulation for nonlinear diffractive imaging is that the data-fidelity term is nonconvex due to the nonlinearity and that the TV regularizer is nondifferentiable.


For such nonsmooth and nonconvex problems, the proximal gradient method, also known as iterative shrinkage/thresholding algorithm (ISTA) \cite{Figueiredo.Nowak2003,Daubechies.etal2004,Bect.etal2004}, is a natural choice. ISTA is easy to implement and is proved to converge under some technical conditions. However, its convergence speed is slow. FISTA \cite{Beck.Teboulle2009} is an accelerated variant of ISTA, which is proved to have the optimal worst case convergence rate for convex problems. Unfortunately, its convergence analysis for nonconvex problems has not been established.

\subsection{Contributions}

In this paper, we propose a new image reconstruction method called \emph{Convergent Inverse Scattering using Optimization and Regularization (CISOR)} for fast and memory-efficient nonlinear diffractive imaging.

The key contributions of this paper are as follows:
\begin{itemize}
\item A novel nonconvex formulation that precisely models the fundamental nonlinear relationship between the scattered wave and the permittivity contrast, while enabling fast and memory-efficient implementation, as well as rigorous convergence analysis.  
\item A new relaxed variant of FISTA for solving our nonconvex problem with rigorous convergence analysis. Our new variant of FISTA may be of interest on its own as a general nonconvex solver.
\item Extension of the proposed formulation and algorithm, as well as the convergence analysis, to the 3D vectorial case, which makes our method applicable in a broad range of engineering and scientific areas.  
\end{itemize} 

\subsection{Related Work}
Many methods that attempt to integrate the nonlinearity of wave scattering have been proposed in the literature. The iterative linearization (IL) method \cite{Belkebir.etal2005,Chaumet.Belkebir} iteratively computes the forward model using the current estimated permittivity, and estimates the permittivity using the field from the previously computed forward model. Hence, each sub-problem at each iteration is a linear problem.
Contrast source inversion (CSI) \cite{vandenBerg.Kleinman1997, Abubakar.etal2005,Bevacquad.etal2017} defines an auxiliary variable called the contrast source, which is the product of the permittivity contrast and the field. CSI alternates between estimating the contrast source and the permittivity. 
Hybrid methods (HM)~\cite{Belkebir.Sentenac2003, Mudry.etal2012, Zhang.etal2016} combine IL and CSI, aiming to benefit from each of their advantages. A comprehensive comparison of these three methods can be found in the review paper \cite{Mudry.etal2012}. Recently, the idea of neural network unfolding has inspired a class of methods that updates the estimates using error backpropagation \cite{Kamilov.etal2015, Kamilov.etal2016, Kamilov.etal2016a, Liu.etal2017, Liu.etal2017b}. While such methods  can in principle model the precise nonlinearity, in practice, the accuracy may be limited by the availability of memory to store the iterates needed to perform unfolding.

Figure \ref{fig:SheppLogan2D} provides a visual comparison of the reconstructed images obtained by the first Born (FB) linear approximation \cite{Born.Wolf2003}, the iterative linearization (IL) method \cite{Belkebir.etal2005,Chaumet.Belkebir}, and our proposed nonlinear method CISOR; detailed experimental setup will be presented in Section \ref{Sec:ExperimentalResults}. We can see that when the contrast of the object is small, all methods achieve similar reconstruction quality. As the contrast value increases, the performance of the linear method degrades significantly and the iterative linearization method only succeeds up to a certain level, whereas our nonlinear method CISOR provides reliable reconstruction for all tested contrast values.

While we were concluding this manuscript, we became aware of very recent related work in \cite{Soubies.etal2017}, which considered a similar problem as in this paper. Our work differs from \cite{Soubies.etal2017} in the following aspects: (\emph{i}) Only the 2D case has been considered in \cite{Soubies.etal2017}, whereas our work extends to the 3D vectorial case. (\emph{ii}) FISTA \cite{Beck.Teboulle2009}, which does not have convergence guarantee for nonconvex problems, is applied in \cite{Soubies.etal2017} as the nonconvex solver, whereas our method applies our new variant of FISTA with rigorous convergence analysis established here.

\section{Problem Formulation}
\label{Sec:formulation}

\begin{figure}[t!]
\centering
\includegraphics[width=8.5cm]{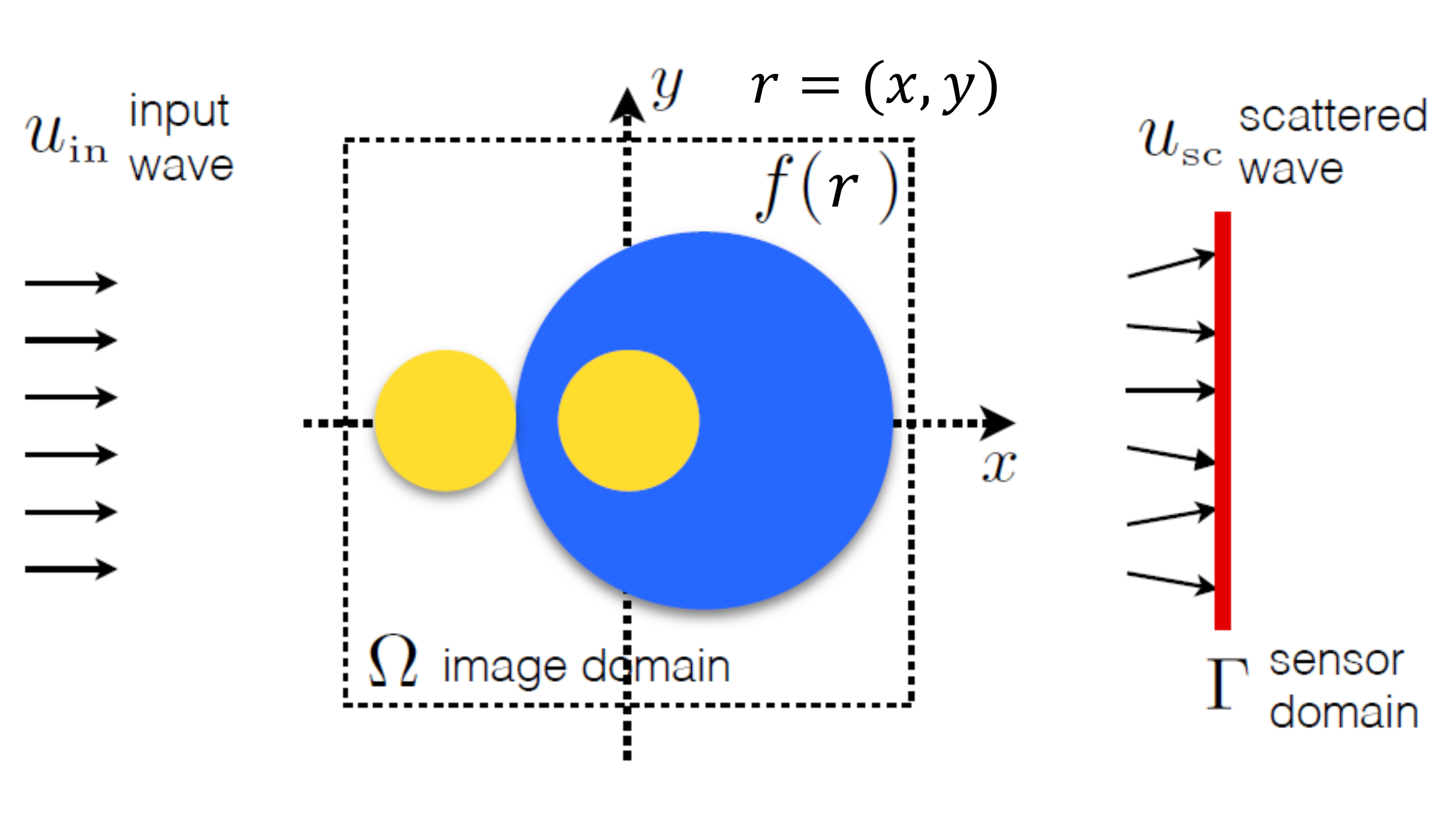}
\caption{Visual representation of the measurement scenario considered in this paper. An object with a real permittivity contrast $f(\rbf)$ is illuminated with an input wave $u_{\text{in}}(\rbf)$, which interacts with the object and results in the scattered wave $u_{\text{sc}}$ at the sensor domain $\Gamma \subset \R^2$. The complex scattered wave is captured at the sensor and the algorithm proposed here is used for estimating the contrast $f$.}
\label{fig:SetUp}
\vspace{-3mm}
\end{figure}

The problem of inverse scattering based on the scalar theory of diffraction \cite{Born.Wolf2003,Goodman1996} is described as follows and illustrated in Figure~\ref{fig:SetUp}; the formulation for the 3D vectorial case is presented in Appendix \ref{app:vectorialCase}. Let $d=2$ or $3$,
suppose that an object is placed within a bounded domain $\Omega\subset\mathbb{R}^d$. The object is illuminated by an incident wave $u_\text{in}$, and the scattered wave $u_{\text{sc}}$ is measured by the sensors placed in a sensing region $\Gamma\subset\mathbb{R}^d$. Let $u$ denote the total field, which satisfies $u(\rbf)=u_{\text{in}}(\rbf)+u_{\text{sc}}(\rbf), \forall \rbf\in\mathbb{R}^d$. The scalar Lippmann-Schwinger equation \cite{Born.Wolf2003} establishes the relationship between wave and permittivity contrast:
\begin{equation*}
u(\rbf) = u_{\text{in}}(\rbf) + k^2\int_{\Omega} g(\rbf-\rbf') u(\rbf') f(\rbf') d\rbf',\quad \forall \rbf \in \mathbb{R}^d.
\end{equation*}
In the above, $f(\rbf)=(\epsilon(\rbf)-\epsilon_b)$ is the permittivity contrast, where $\epsilon(\rbf)$ is the permittivity of the object, $\epsilon_b$ is the permittivity of the background, and $k=2\pi/\lambda$ is the wavenumber in vacuum. We assume that $f$ is real, or in other words, the object is lossless. The free-space Green's function is defined as follows:
\begin{equation}
g(\rbf)= 
\begin{cases}
-\frac{j}{4}H_0^{(1)}(k_b \|\rbf\|),&\text{if } d=2\\
\frac{e^{jk_b\|\rbf\|}}{4\pi\|\rbf\|}, &\text{if } d=3 
\end{cases},\label{eq:scalarGreens}
\end{equation}
where $\|\cdot\|$ denotes the Euclidean norm, $H_0^{(1)}$ is the Hankel function of first kind, and $k_b=k\sqrt{\epsilon_b}$ is the wavenumber of the background medium. The corresponding discrete system is then
\begin{align}
\ybf &= \Hbf \diag{\ubf}\fbf + \ebf,\label{eq:discrete2D_measurement}\\
\ubf &= \ubf^{\text{in}} + \Gbf\diag{\fbf} \ubf, \label{eq:discrete2D_constraint}
\end{align}
where $\fbf\in\mathbb{R}^N$, $\ubf\in\mathbb{C}^N$, $\ubf^{\text{in}}\in\mathbb{C}^N$ are $N$ uniformly spaced samples of $f(\rbf)$, $u(\rbf)$, and $u_{\text{in}}(\rbf)$ on $\Omega$, respectively, and $\ybf\in\mathbb{C}^M$ is the measured scattered wave at the sensors with measurement error $\ebf\in\mathbb{C}^M$. For a vector $\abf\in\mathbb{R}^N$, $\diag{\abf}\in\mathbb{R}^{N\times N}$ is a diagonal matrix with $\abf$ on the diagonal. The matrix $\Hbf\in\mathbb{C}^{M\times N}$ is the discretization of Green's function $g(\rbf-\rbf')$ with $\rbf\in\Gamma$ and $\rbf'\in\Omega$, whereas $\Gbf\in\mathbb{C}^{N\times N}$ is the discretization of Green's function with $\rbf,\rbf'\in \Omega$. The inverse scattering problem is then to estimate $\fbf$ given $\ybf$, $\Hbf$, $\Gbf$, and $\ubf^{\text{in}}$. 
Define
\begin{equation}
\Abf:=\Ibf - \Gbf\diag{\fbf}.
\label{eq:A_def}
\end{equation}
Note that \eqref{eq:discrete2D_measurement} and \eqref{eq:discrete2D_constraint} define a nonlinear inverse problem, because $\ubf$ depends on $\fbf$ through $\ubf = \Abf^{-1}\ubf^{\text{in}}$, where $\Abf$ is defined in \eqref{eq:A_def}. In this paper, the linear method refers to the formulation where $\ubf=\ubf^{\text{in}}$, and the iterative linearization method replaces $\ubf$ in \eqref{eq:discrete2D_measurement} with the estimate $\widehat{\ubf}$ computed from \eqref{eq:discrete2D_constraint} using the current estimate $\widehat{\fbf}$ and assumes that $\widehat{\ubf}$ is a constant with respect to $\fbf$.

To estimate $\fbf$ from the nonlinear inverse problem \eqref{eq:discrete_measurement} and \eqref{eq:discrete_constrain}, we consider the following nonconvex optimization formulation. Define
\begin{equation}
\mathcal{Z}(\fbf):=\Hbf\diag{\ubf}\fbf,\label{eq:Z_def}
\end{equation}
which is the (clean) scattered wave from the object with permittivity contrast $\fbf$.
Moreover, let $\mathcal{C}\subset\mathbb{R}^N$ be a bound convex set that contains all possible values that $\fbf$ may take.
Then $\fbf$ is estimated by minimizing the following composite cost function with a nonconvex data-fidelity term $\Dcal(\fbf)$ and a convex regularization term $\Rcal(\fbf)$:
\begin{equation}
\fbf^* = \arg\min_{\fbf\in\mathbb{R}^N} \left\{\mathcal{F}(\fbf):=\Dcal(\fbf) + \Rcal(\fbf)\right\},
\label{eq:objective_func}
\end{equation}
where
\begin{align}
\Dcal(\fbf)&=\frac{1}{2}\|\ybf - \mathcal{Z}(\fbf)\|_2^2\label{eq:D_def}, \\
\Rcal(\fbf)&= \tau \sum_{n=1}^N \sqrt{\sum_{d=1}^2 | [\Dbf_d \fbf]_n|^2} + \chi_\mathcal{C}(\fbf).\label{eq:R_def}
\end{align}
In \eqref{eq:R_def}, $\Dbf_d$ is the discrete gradient operator in the $d$th dimension, hence the first term in $\Rcal(\fbf)$ is the total variation (TV) cost, and the parameter $\tau>0$ controls the contribution of the TV cost to the total cost. The second term $\chi_\mathcal{C}(\cdot)$ is defined as
\begin{align*}
\chi_\mathcal{C}(\fbf):=\begin{cases}
0,&\text{if }\fbf\in \mathcal{C}\\
\infty, &\text{if }\fbf\not\in \mathcal{C}
\end{cases}.
\end{align*}
Note that $\mathcal{D}(\cdot)$ is differentiable if $\Abf$ is non-singular, $\mathcal{R}(\cdot)$ is proper, convex, and lower semi-continuous if $\mathcal{C}$ is convex and closed.


\section{Proposed Method}
\label{Sec:ProposedMethod}

As mentioned in Section \ref{Sec:Intro} that for a cost function like $\eqref{eq:objective_func}$, the class of proximal gradient methods, including ISTA \cite{Figueiredo.Nowak2003,Daubechies.etal2004,Bect.etal2004} and FISTA \cite{Beck.Teboulle2009}, can be applied. However, ISTA is empirically slow and FISTA has only been proved to converge for convex problems.
A variant of FISTA has been proposed in \cite{Li.Lin2015} for nonconvex optimization with convergence guarantees. This algorithm computes two estimates from ISTA and FISTA, respectively, at each iteration, and selects the one with lower objective function value as the final estimate at the current iteration. Therefore, both the gradient and the objective function need to be evaluated at two different points at each iteration. While such extra computation may be insignificant in some applications, it can be prohibitive in the inverse scattering problem, where additional evaluations of the gradient and the objective function require the computation of the entire forward model. Another accelerated proximal gradient method that is proved to converge for nonconvex problems is proposed in \cite{Ghadimi.Lan2016}, which is not directly related to FISTA for non-smooth objective functions like \eqref{eq:objective_func}.

\subsection{Relaxed FISTA}
\label{subsec:relaxedFISTA}

We now introduce our new variant of FISTA to solve \eqref{eq:objective_func}. 
Starting with some initialization $\fbf_0\in\mathbb{R}^N$ and setting $\sbf_1=\fbf_0$, $t_0=1$, $\alpha\in[0,1)$, for $k\geq 1$, the proposed algorithm proceeds as follows:
\begin{align}
\fbf_k &= \prox_{\gamma\mathcal{R}} \left(\sbf_k - \gamma \grad \Dcal(\sbf_k)\right)\label{eq:algo1}\\
t_{k+1} &= \frac{\sqrt{4t_k^2+1}+1}{2}\label{eq:algo2}\\
\sbf_{k+1} &=\fbf_k + \alpha\left(\frac{t_k-1}{t_{k+1}}\right)(\fbf_k-\fbf_{k-1})\label{eq:algo3},
\end{align}
where the choice of the step-size $\gamma$ to ensure convergence will be discussed in Section \ref{subsec:TheoreticalAnalysis}. Notice that the algorithm \eqref{eq:algo1}-\eqref{eq:algo3} is equivalent to ISTA when $\alpha=0$ and is equivalent to FISTA when $\alpha=1$. For this reason, we call it relaxed FISTA. Figure \ref{fig:compAlpha} shows that the empirical convergence speed of relaxed FISTA improves as $\alpha$ increases from $0$ to $1$. The plot was obtained by using the experimentally measured scattered microwave data collected by the Fresnel Institute \cite{Geffrin.etal2005}. Our theoretical analysis of relaxed FISTA in Section \ref{subsec:TheoreticalAnalysis} establishes convergence for any $\alpha \in [0, 1)$ with appropriate choice of the step-size $\gamma$.

\begin{figure}[t!]
\centering
\includegraphics[width=8cm]{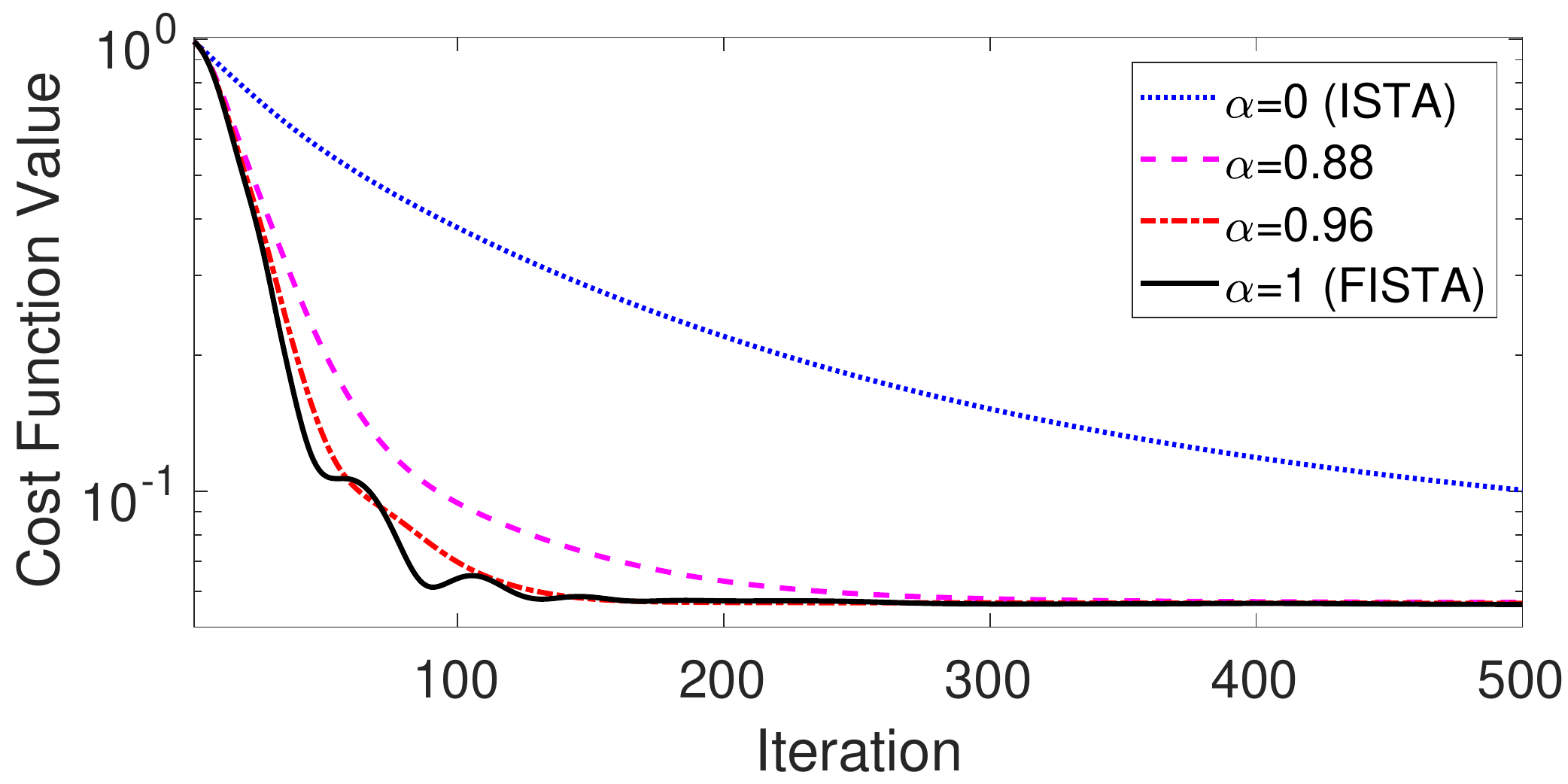}
\caption{Empirical convergence speed for relaxed FISTA with various $\alpha$ values tested on experimentally measured data.}
\label{fig:compAlpha}
\vspace{-3mm}
\end{figure}

The two main elements of relaxed FISTA are the computation of the gradient $\grad \Dcal$ and the proximal mapping $\prox_{\gamma\Rcal}$. Given $\grad\Dcal(\sbf_k)$, the proximal mapping \eqref{eq:algo1} can be efficiently solved \cite{Beck.Teboulle2009a,Kamilov2017}. 
The following proposition provides an explicit formula for $\grad\Dcal$, which enables fast and memory-efficient computation of $\grad\Dcal$.
\begin{proposition}
\label{prop:gradient}
Let $\Zcal(\fbf)$ be defined in \eqref{eq:Z_def} and $\wbf=\Zcal(\fbf)-\ybf$. Then
\begin{equation}
\grad \Dcal(\fbf) = \textsf{Re}\left\{\diag{\ubf}^\H\left(\Hbf^\H\wbf + \Gbf^\H\vbf\right)\right\}, 
\label{eq:gradient}
\end{equation}
where $\ubf$ and $\vbf$ are obtained from the linear systems
\begin{equation}
\Abf\ubf = \ubf^\text{in},\quad\text{and}\quad \Abf^\H\vbf=\diag{\fbf}\Hbf^\H\wbf.
\label{eq:u_and_v}
\end{equation}
\end{proposition}
\begin{proof}
See Appendix \ref{app:prop_grad_proof}.
\end{proof}
In the above, $\ubf$ and $\vbf$ can be efficiently solved by the conjugate gradient method. Note that the formulation for computing the gradient in Proposition \ref{prop:gradient} is also known as the adjoint state method \cite{Plessix2006}.
In our implementation, $\Abf$ is an operator rather than an explicit matrix, and the convolution with Green's function in $\Abf$ is computed using the fast Fourier transform (FFT) algorithm. Note that the formula for the gradient defined in \eqref{eq:gradient} is for the scalar field case. The formula for the 3D vectorial case is provided in \eqref{eq:3Dgradient} in Appendix \ref{app:vectorialCase}.

\subsection{Theoretical Analysis}
\label{subsec:TheoreticalAnalysis}

We now provide the convergence analysis of relaxed FISTA applied to our nonconvex optimization problem \eqref{eq:objective_func}.

The following proposition shows that the data-fidelity term \eqref{eq:D_def} has Lipschitz gradient on a bounded domain. Note that Lipschitz gradient of the smooth term in a composite cost function like \eqref{eq:objective_func} is essential to prove the convergence of relaxed FISTA, which we will establish in Proposition \ref{prop:converge}. 
\begin{proposition}
\label{prop:lip_grad}
Suppose that $\mathcal{U}\subset\mathbb{R}^N$ is bounded. Assume that $\|\ubf^{\text{in}}\|<\infty$ and the matrix $\Abf$ defined in \eqref{eq:A_def} is non-singular for all $\sbf\in\mathcal{U}$. Then $\Dcal(\sbf)$ has Lipschitz gradient on $\mathcal{U}$. That is, there exists an $L\in(0,\infty)$ such that 
\begin{equation}
\|\grad \Dcal(\sbf_1) - \grad \Dcal(\sbf_2)\| \leq L\|\sbf_1 - \sbf_2\|,\quad\forall \sbf_1,\sbf_2\in \mathcal{U}.
\label{eq:Lips_constant}
\end{equation}
\end{proposition}
\begin{proof}
See Appendix \ref{app:prop_lipschitz_proof}.
\end{proof}
Notice that all $\fbf_k$ obtained from \eqref{eq:algo1} are within a bounded set $\mathcal{C}$, and each $\sbf_{k+1}$ obtained from \eqref{eq:algo3} is a linear combination of $\fbf_k$ and $\fbf_{k-1}$, where the weight $\alpha\left(\frac{t_k-1}{t_{k+1}}\right)\in[0,1)$ since $\alpha\in[0,1)$ and $\frac{t_k-1}{t_{k+1}}\leq 1$ by \eqref{eq:algo2}. Hence, the set that covers all possible values for $\{\fbf_k\}_{k\geq 0}$ and $\{\sbf_k\}_{k\geq 1}$ is bounded. Using this fact, we have the following convergence guarantee for relaxed FISTA applied to solve \eqref{eq:objective_func}.

\begin{proposition}
\label{prop:converge}
Let $\mathcal{U}$ in Proposition \ref{prop:lip_grad} be the bounded set that covers all possible values for $\{\fbf_k\}_{k\geq 0}$ and $\{\sbf_k\}_{k\geq 1}$ obtained from \eqref{eq:algo1} and \eqref{eq:algo3}, $L$ be the corresponding Lipschitz constant defined in \eqref{eq:Lips_constant}. Choose $0<\gamma\leq \frac{1-\alpha^2}{2L}$ for any fixed $\alpha\in [0,1)$. 
Define the gradient mapping as
\begin{equation}
\mathcal{G}_{\gamma}(\fbf):=\frac{\fbf - \prox_{\gamma\Rcal}\left(\fbf - \gamma\grad \Dcal(\fbf)\right)}{\gamma}, \quad \forall \fbf \in\mathbb{R}^N.
\label{eq:gradient_mapping_def}
\end{equation}
Then, relaxed FISTA achieves the stationary points of the cost function $\mathcal{F}$ defined in \eqref{eq:objective_func} in the sense that the gradient mapping norm satisfies
\begin{equation}
\lim_{k\rightarrow\infty}\|\mathcal{G}_{\gamma}(\fbf_k)\| = 0.
\label{eq:gradient_mapping_converge}
\end{equation}
Moreover, 
Let $\mathcal{F}^*$ denote the global minimum of $\mathcal{F}$ and we assume its existence. Then for any $K>0$,
\begin{equation}
\min_{k\in\{1,\ldots,K\}} \|\mathcal{G}_{\gamma}(\fbf_k)\|^2 \leq \frac{2L(3+\gamma L)^2\left(\mathcal{F}(\fbf_0) - \mathcal{F}^*\right)}{K\gamma L(1 - \gamma L)}.\label{eq:convergence_rate}
\end{equation}
\end{proposition}
\begin{proof}
See Appendix \ref{app:prop_converge_proof}.
\end{proof}
Note that for any $\fbf\in\mathbb{R}^N$, $\mathcal{G}_{\gamma}(\fbf)=0$ implies $0\in\partial\mathcal{F}(\fbf)$, where $\partial\mathcal{F}(\fbf)$ is the limiting subdifferential \cite{Rockafellar2009} of $\mathcal{F}$ defined in \eqref{eq:objective_func} at $\fbf$, hence $\fbf$ is a stationary point of $\mathcal{F}$.

Relaxed FISTA can be used as a general nonconvex solver for any cost function that has a smooth nonconvex term with Lipschitz gradient and a nonsmooth convex term whose proximal mapping can be easily computed. The convergence analysis of relaxed FISTA does not require the estimates to be constrained on a bounded domain. The condition of bounded $\mathcal{U}$ in the statement of Proposition \ref{prop:converge} is to ensure that the gradient of the specific function $\Dcal(\fbf)$ defined in \eqref{eq:D_def} is Lipschitz, as discussed in Proposition \ref{prop:lip_grad}.


\section{Experimental Results}
\label{Sec:ExperimentalResults}

We now compare our method CISOR with several state-of-the-art methods, including iterative linearization (IL) \cite{Belkebir.etal2005,Chaumet.Belkebir}, contrast sourse inversion (CSI) \cite{vandenBerg.Kleinman1997, Abubakar.etal2005,Bevacquad.etal2017}, and SEAGLE \cite{Liu.etal2017}, as well as a conventional linear method, the first Born approximation (FB) \cite{Born.Wolf2003}. All algorithms use additive total variation regularization. In our implementation, CSI uses Polak-Ribi\'{e}re conjugate gradient. CISOR uses the relaxed FISTA defined in Section\ref{subsec:relaxedFISTA} with $\alpha=0.96$ and fixed step-size $\gamma$, which is manually tuned.  The other methods use the standard FISTA \cite{Beck.Teboulle2009}, also with manually tuned and fixed step-sizes.

\begin{figure}[t!]
\centering
\includegraphics[width=8cm]{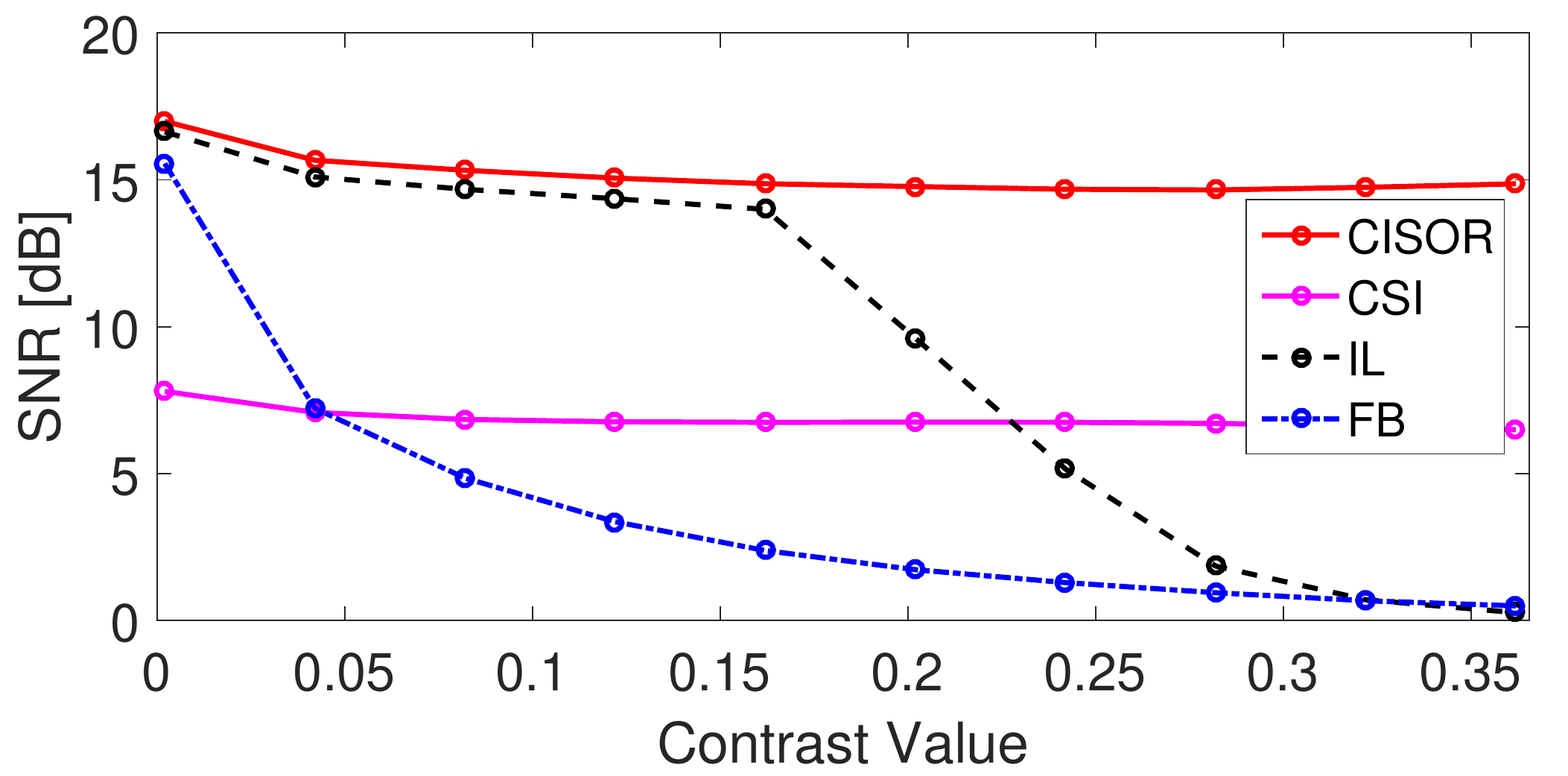}
\caption{Comparison of different reconstruction methods for various contrast levels tested on simulated data.}
\label{fig:compContrast}
\vspace{-3mm}
\end{figure}

\begin{figure*}
\centering
\includegraphics[width=16cm]{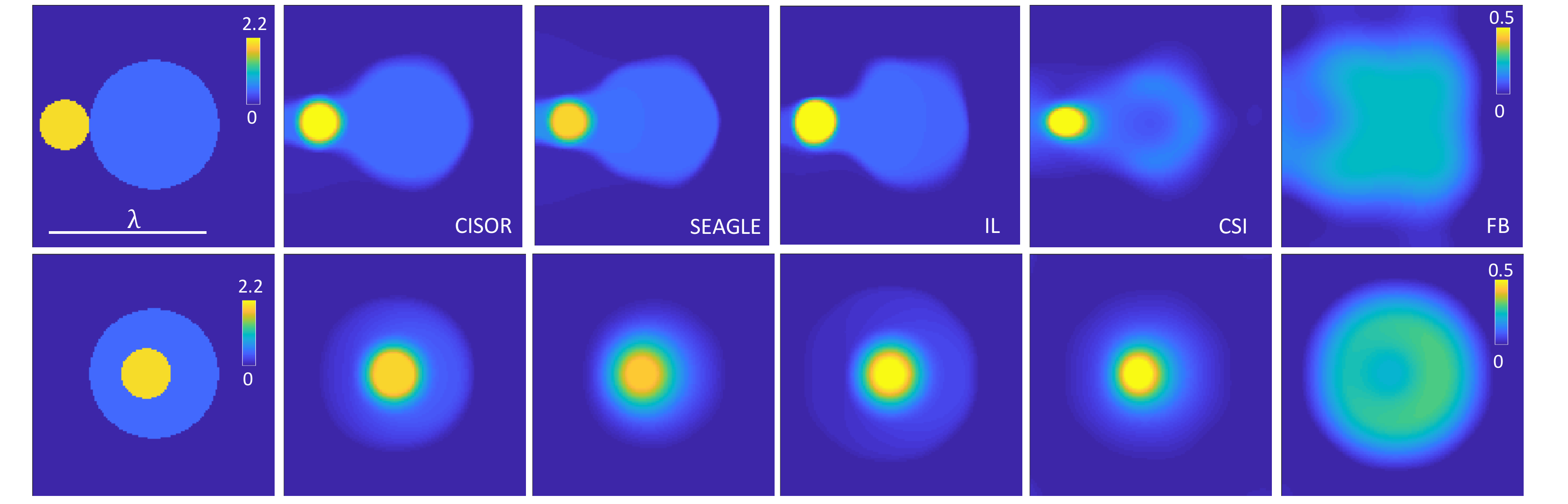}
\caption{Reconstructed images obtained by different algorithms from 2D experimentally measured data. The first and second rows use the \emph{FoamDielExtTM} and the \emph{FoamDielIntTM} objects, respectively. From left to right: ground truth, reconstructed images by CISOR, SEAGLE, IL, CSI, and FB. The color-map for FB is different from the rest, because FB significantly underestimated the contrast value. The size of the reconstructed objects are $128\times 128$ pixels.}
\label{fig:Fresnel2D}
\vspace{-3mm}
\end{figure*}

\begin{figure*}
\centering
\includegraphics[width=16cm]{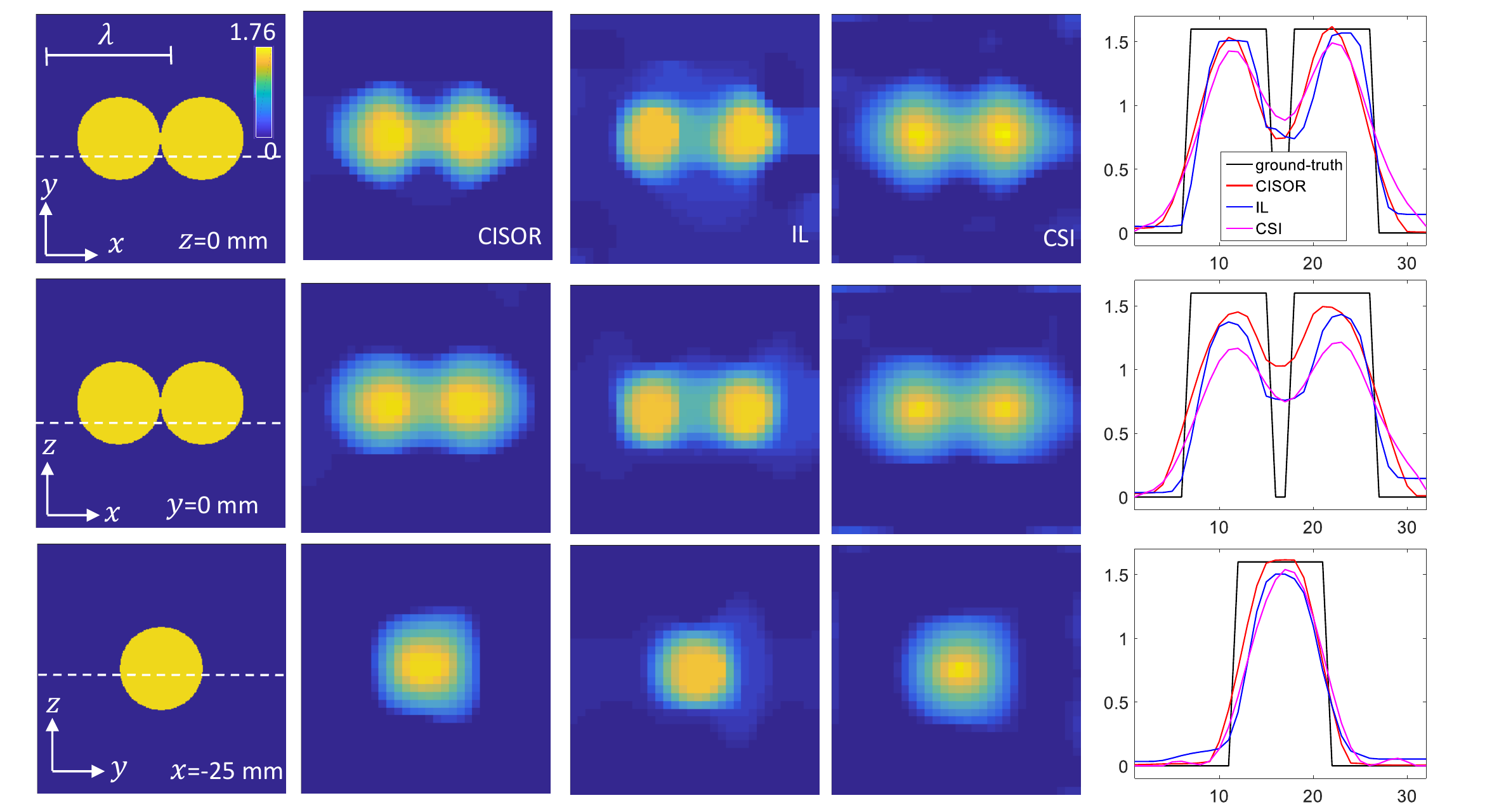}
\caption{Reconstructed images obtained by CISOR, IL, and CSI from 3D experimentally measured data for the \emph{TwoShperes} object.  From left to right: ground truth, reconstructed image slices by CISOR, IL, and CSI, the reconstructed contrast distribution along the dashed lines showing in the image slices on the first column. From top to bottom: image slices parallel to the $x$-$y$ plane with $z=0$ mm, parallel to the $x$-$z$ plane with $y=0$ mm, and parallel to the $y$-$z$ plane with $x=-25$ mm. The size of the reconstructed objects are $32\times 32\times 32$ pixels for a $150\times 150\times 150$ mm cube centered at $(0,0,0)$.} 
\label{fig:Fresnel3DTwoSpheres}
\vspace{-3mm}
\end{figure*}

\textbf{Comparison on simulated data.}
The wavelength of the incident wave in this experiment is 7.49 cm.
Define the contrast of an object with permittivity contrast distribution $\fbf$ as $\max(|\fbf|)$. We consider the Shepp-Logan phantom and change its contrast to the desired value to obtain the ground-truth $\fbf_{\text{true}}$. 
We then solve the Lippmann-Schwinger equation to generate the scattered waves that are then used as measurements. 
The center of the image is the origin and the physical size of the image is set to 120 cm $\times$ 120 cm.
Two linear detectors are placed on two opposite sides of the image at a distance of 95.9 cm from the origin. Each detector has 169 sensors with a spacing of 3.84 cm. The transmitters are placed on a line 48.0 cm to the left of the left detector, and they are spaced uniformly in azimuth with respect to the origin within a range of $[-60^\circ,60^\circ]$ at every $5^\circ$.
The reconstructed SNR, which is defined as $20\log_{10}(\|\fbf_\text{true}\|/\|\hat{\fbf}-\fbf_\text{true}\|)$, is used as the comparison criterion. The size the reconstructed images $\hat{\fbf}$ is $128\times 128$ pixels. For each contrast value and each algorithm, we run the algorithm with five different regularization parameter values and select the result that yields the highest reconstructed SNR. 

Figure \ref{fig:compContrast} shows that as the contrast increases, the reconstructed SNR of FB and IL decreases, whereas that of CSI and CISOR is more stable. While it is possible to further improve the reconstructed SNR for CSI by running more iterations, CSI is known to be slow as the comparisons shown in \cite{Mudry.etal2012}.  A visual comparison between FB, IL, and CISOR at several contrast levels has been presented in Figure \ref{fig:SheppLogan2D} at the beginning of the paper.

\begin{figure*}
\centering
\includegraphics[width=16cm]{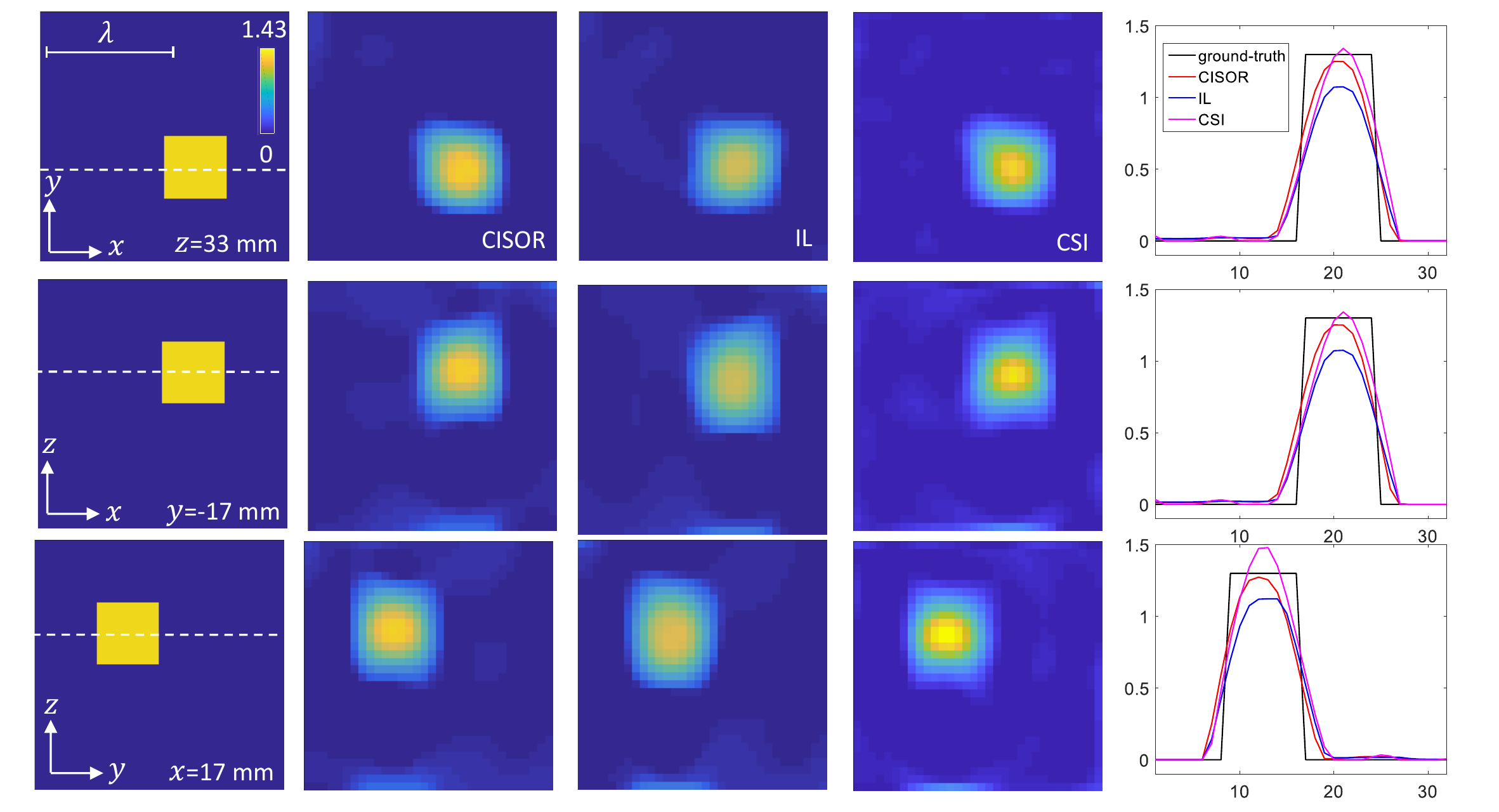}
\caption{Reconstructed images obtained by CISOR, IL, and CSI from 3D experimentally measured data for the \emph{TwoCubes} object.  From left to right: ground truth, reconstructed image slices by CISOR, IL, and CSI, the reconstructed contrast distribution along the dashed lines showing in the image slices on the first column. From top to bottom: image slices parallel to the $x$-$y$ plane with $z=33$ mm, parallel to the $x$-$z$ plane with $y=-17$ mm, and parallel to the $y$-$z$ plane with $x=17$ mm. The size of the reconstructed objects are $32\times 32\times 32$ pixels for a $100\times 100\times 100$ mm cube centered at $(0,0,50)$ mm.} 
\label{fig:Fresnel3DTwoCubes}
\vspace{-3mm}
\end{figure*}

\textbf{Comparison on experimental data.}
We test our method in both 2D and 3D settings using the public dataset provided by the Fresnel Institute \cite{Geffrin.etal2005,Geffrin.Sabouroux2009}.
Two objects for the 2D setting, \emph{FoamDielExtTM} and \emph{FoamDielintTM}, and two objects for the 3D setting, \emph{TwoSpheres} and \emph{TwoCubes}, are considered.

In the 2D setting, the objects are placed within a 150 mm $\times$ 150 mm square region centered at the origin of the coordinate system. The number of transmitters is 8 and the number of receivers is 360 for all objects. The transmitters and the receivers are placed on a circle centered at the origin with radius 1.67 m and are spaced uniformly in azimuth. Only one transmitter is turned on at a time and only 241 receivers are active for each transmitter. That is, the 119 receivers that are closest to a transmitter are inactive for that transmitter.  While the dataset contains multiple frequency measurements, we only use the ones corresponding to 3 GHz, hence the wavelength of the incident wave is 99.9 mm. The pixel size of the reconstructed images is 1.2 mm.

In the 3D setting, the transmitters are located on a sphere with radius 1.769 m. The azimuthal angle $\theta$ ranges from $20^\circ$ to $340^\circ$ with a step of $40^\circ$, and the polar angle $\phi$ ranges from $30^\circ$ to $150^\circ$ with a step of $15^\circ$.  The receivers are only placed on a circle with radius 1.769 m in the azimuthal plane with azimuthal angle ranging from $0^\circ$ to $350^\circ$ with a step of $10^\circ$. Only the receivers that are more than $50^\circ$ away from a transmitter are active for that transmitter. A visual representation of this setup is shown in Figure \ref{fig:3DSetUp} in the Appendix. We use the data corresponding to 4 GHz for the \emph{TwoSpheres} object and 6 GHz for the \emph{TwoCubes} object, hence the wavelength of the incident wave is  74.9 mm and 50.0 mm, respectively. The pixel size is 4.7 mm for the \emph{TwoSpheres} and 3.1 mm for \emph{TwoCubes}.

Figure \ref{fig:Fresnel2D} provides a visual comparison of the reconstructed images obtained by different algorithms for the 2D data. For each object and each algorithm, we run the algorithm with five different regularization parameter values and select the result that has the best visual quality. 
Figure \ref{fig:Fresnel2D} shows that all nonlinear methods CISOR, SEAGLE, IL, and CSI obtained reasonable reconstruction results in terms of both the contrast value and the shape of the object, whereas the linear method FB significantly underestimated the contrast value and failed to capture the shape. These results demonstrate that the proposed method is competitive with several state-of-the-art methods. 

Figures \ref{fig:Fresnel3DTwoSpheres} and \ref{fig:Fresnel3DTwoCubes} present the results for the \emph{TwoSpheres} and the \emph{TwoCubes} objects, respectively. Again, the results show that CISOR is competitive with the state-of-the-art methods for the 3D vectorial setting as well.

\section{Conclusion}
\label{Sec:Conclusion}

In this paper, we proposed a nonconvex formulation for nonlinear diffractive imaging based on the scalar theory of diffraction, and further extended our formulation to the 3D vectorial field setting. The nonconvex optimization problem was solved by our new variant of FISTA. We provided an explicit formula for fast computation of the gradient at each FISTA iteration and proved that the algorithm converges for our nonconvex problem. Numerical results demonstrated that the proposed method is competitive with several state-of-the-art methods. 
The advantages of CISOR over other methods are mainly in the following two aspects. 
First, CISOR is more memory-efficient due to the explicit formula for the gradient as provided in Proposition \ref{prop:gradient}. The formula allows using the conjugate gradient method to accurately compute the gradient without the need of storing all the iterates, which was required in SEAGLE \cite{Liu.etal2017b}.
Second, CISOR enjoys rigorous theoretical convergence analysis as established in Proposition \ref{prop:converge}, while other methods only have reported empirical convergence performance.


\section{Appendix}
\label{Sec:Appendix}

\subsection{3D Diffractive Imaging}
\label{app:vectorialCase}

\subsubsection{Problem Formulation}
\label{app:3DFormulation}

The measurement scenario for the 3D case follows that in \cite{Geffrin.Sabouroux2009} and is illustrated in Figure \ref{fig:3DSetUp}.
The fundamental object-wave relationship in case of electromagnetic wave obeys Maxwell's equations. For time-harmonic electromagnetic field and under the Silver-M\"{u}ller radiation condition, it can be shown that the solution to Maxwell's equations is equivalent to that of the following integral equation \cite{Colton.Kress1992}:
\begin{equation}
\vec{E}(\rbf) = \vec{E}^{\text{in}}(\rbf) + (k^2\Ibf+\grad\divergence)\int_ {\Omega} g(\rbf-\rbf')f(\rbf')\vec{E}(\rbf') d\rbf', \label{eq:E_total} 
\end{equation}
which holds for all $\rbf\in\mathbb{R}^3$.
In \eqref{eq:E_total}, $\vec{E}(\rbf)\in\mathbb{C}^3$ is the electric field at spatial location $\rbf$, which is the sum of the incident field $\vec{E}^{\text{in}}(\rbf)\in\mathbb{C}^3$ and the scattered field $\vec{E}^{\text{sc}}(\rbf)\in\mathbb{C}^3$. The scalar permittivity contrast distribution is defined as $f(\rbf)=(\epsilon(\rbf)-\epsilon_b)$, where $\epsilon(\rbf)$ is the permittivity of the object, $\epsilon_b$ is the permittivity of the background, and $k=2\pi/\lambda$ is the wavenumber in vacuum. We assume that $f(\rbf)$ is real. The 3D free-space scalar Green's function $g(\rbf)$ is defined in \eqref{eq:scalarGreens}. 

\begin{figure}
\includegraphics[width=0.5\textwidth]{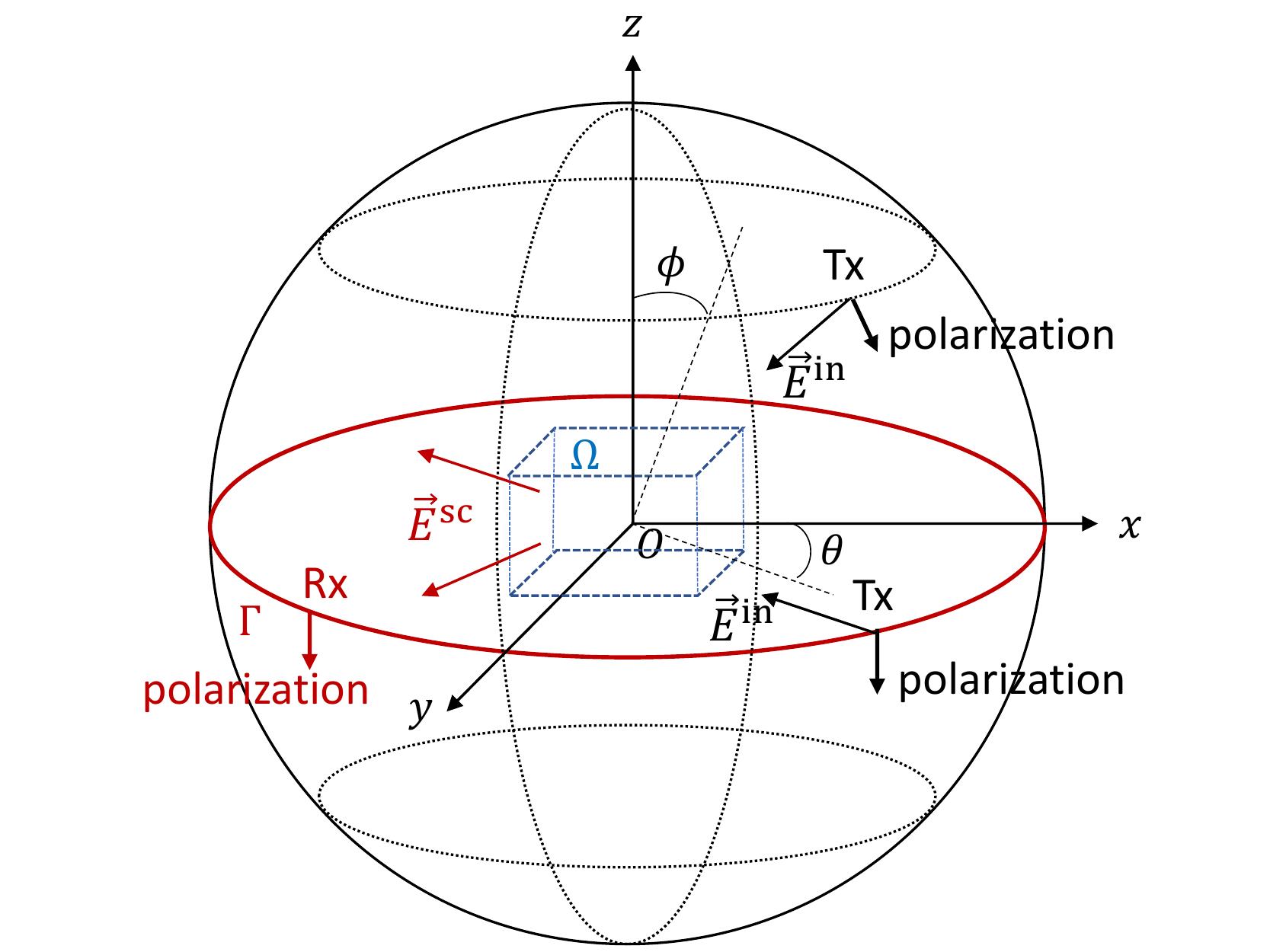}
\caption{The measurement scenario for the 3D case considered in this paper. The object is placed within a bounded image domain $\Omega$. The transmitter antennas (Tx) are placed on a sphere and are linearly polarized. The arrows in the figure define the polarization direction. The receiver antennas (Rx) are placed in the sensor domain $\Gamma$ within the $x$-$y$ (azimuth) plane, and are linearly polarized along the $z$ direction.}
\label{fig:3DSetUp}
\vspace{-3mm}
\end{figure}

As illustrated in Figure~\ref{fig:3DSetUp}, the measurements in our problem are the scattered field measured in the sensor region $\Gamma$,
\begin{equation}
\vec{E}^{\text{sc}}(\rbf) = (k^2\Ibf + \grad\divergence)\int_{\Omega} g(\rbf-\rbf')f(\rbf')\vec{E}(\rbf')d\rbf', \forall \rbf\in\Gamma, \label{eq:E_scatter1}
\end{equation}
where the total field $\vec{E}$ in \eqref{eq:E_scatter1} is obtained by evaluating \eqref{eq:E_total} at $\rbf\in\Omega$.
In the experimental setup considered in this paper, $\Gamma\cap\Omega=\emptyset$, hence Green's function is non-singular within the integral region in \eqref{eq:E_scatter1}. Therefore, we can conveniently move the gradient-divergence operator $\grad\divergence$ inside the integral. Then \eqref{eq:E_scatter1} becomes
\begin{equation}
\vec{E}^{\text{sc}}(\rbf) = \int_{\Omega} \tensor{G}(\rbf-\rbf')f(\rbf')\vec{E}(\rbf')d\rbf', \forall \rbf\in\Gamma, \label{eq:E_scatter2}
\end{equation}
where $\tensor{G}(\rbf-\rbf')=(k^2\Ibf + \grad\grad) g(\rbf-\rbf')$ is the dyadic Green's function in free-space, which has an explicit form:
\begin{align}
\tensor{G}(\rbf-\rbf')&=k^2\left(\left(\frac{3}{k^2d^2}-\frac{3j}{kd}-1\right)\left(\frac{\rbf-\rbf'}{d} \otimes \frac{\rbf-\rbf'}{d} \right)\right.\nonumber\\
&\left.+\left(1+\frac{j}{kd}-\frac{1}{k^2d^2}\right)\tensor{I}\right)g(\rbf-\rbf'),
\label{eq:dyadicGreen}
\end{align}
where $\otimes$ denotes the Kronecker product, $d=\|\rbf-\rbf'\|$, and $\tensor{I}$ is the unit dyadic.

\subsubsection{Discretization}

To obtain a discrete system for \eqref{eq:E_total} and \eqref{eq:E_scatter2}, we define the image domain $\Omega$ as a cube and  uniformly sample $\Omega$ on a rectangular grid with sampling step $\delta$ in all three dimensions. Let the center of $\Omega$ be the origin and let $\rbf_{l,s,t}:=((l-J/2-0.5)\delta,(s-J/2-0.5)\delta,(t-J/2-0.5)\delta)$ for $r,s,t=1,\ldots,J$. To simplify the notation, we use a one-to-one mapping $(l,s,t)\mapsto n$ and denote the samples by $\rbf_n$ for $n=1,\ldots,N$, where $N=J^3$.  Assume that $M$ detectors are placed at $\rbf_m$ for $m=1,\ldots,M$. Note that the detectors do not have to be placed on a regular grid, because the dyadic Green's function \ref{eq:dyadicGreen} can be evaluated at any spatial points, whereas the rectangular grid on $\Omega$ is important in order for discrete differentiation to be easily defined. Using these definitions, the discrete system that corresponds to \eqref{eq:E_scatter2} and \eqref{eq:E_total} with $\rbf\in\Omega$ can then be written as
\begin{align}
\vec{E}^{\text{sc}}(\rbf_m) &=\delta^3\sum_{n=1}^{N} \tensor{G}(\rbf_m-\rbf_n)f(\rbf_n)\vec{E}(\rbf_n), 
\label{eq:discrete0_measurement}\\
\vec{E}(\rbf_n) &=\vec{E}^{\text{in}}(\rbf_n)+(k^2+\grad\divergence)\vec{B}(\rbf_n),
\label{eq:discrete0_constrain}
\end{align}
where $m=1,\ldots,M$, $n=1,\ldots,N$, and
\begin{equation*}
\vec{B}(\rbf_n)=\delta^3\sum_{k=1}^{N} g(\rbf_n-\rbf_k)f(\rbf_k)\vec{E}(\rbf_k),\quad n=1,\ldots,N.
\end{equation*}

Let us organize $\vec{E}_n$ and $\vec{E}_n^{\text{in}}$ for $n=1,\ldots,N$ into column vectors as follows:
\begin{equation}
\widetilde{\ubf} = \left(
\begin{matrix}
E^{(1)}\\
E^{(2)}\\
E^{(3)}
\end{matrix}\right),\qquad \widetilde{\ubf}^{\text{in}} = \left(
\begin{matrix}
E^{\text{in},(1)}\\
E^{\text{in},(2)}\\
E^{\text{in},(3)}
\end{matrix}\right),
\label{eq:u_def}
\end{equation}
where $E^{(i)}\in\mathbb{C}^N$ is a column vector whose $n^\text{th}$ coordinate is $E^{(i)}_n$; similar notation applies to $E^{\text{in},(i)}\in\mathbb{C}^N$. Then the discretized inverse scattering problem is defined as follows:
\begin{align}
\widetilde{\ybf} &=\widetilde{\Hbf} \left(\Ibf_3\otimes\diag{\widetilde{\fbf}}\right)\widetilde{\ubf} + \widetilde{\ebf},\label{eq:discrete_measurement}\\
\widetilde{\ubf} & = \widetilde{\ubf}^{\text{in}} + (k^2\Ibf + \Dbf)\left(\Ibf_3\otimes (\widetilde{\Gbf}\diag{\widetilde{\fbf}})\right)\widetilde{\ubf},
\label{eq:discrete_constrain}
\end{align} 
where $\Ibf_p$ is a $p\times p$ identity matrix and we drop the subscript $p$ if the dimension is clear from the context, $\widetilde{\fbf}\in\mathbb{R}^N$ is the vectorized permittivity contrast distribution, which is assumed to be real in this work, $\Dbf\in\mathbb{R}^{3N\times 3N}$ is the matrix representation of the gradient-divergence operator $\grad\divergence$, $\widetilde{\ybf}\in\mathbb{C}^M$ is the scattered wave measurement vector with measurement noise $\widetilde{\ebf}\in\mathbb{C}^M$, and $\widetilde{\Gbf}\in\mathbb{C}^{N\times N}$ is the matrix representation of the convolution operator induced by the 3D free-space Green's function \eqref{eq:scalarGreens}.
Note that according to the polarization of the receiver antennas in the experimental setup we consider in this paper, the ideal noise-free measurements $\widetilde{y}_m-\widetilde{e}_m$ are $E^{\text{sc},(3)}_m$, for $m=1,...,M$, which are the scattered wave $\vec{E}^{\text{sc}}$ along the $z$-dimension measured at $M$ different locations in the sensor region $\Gamma$. Therefore, $\widetilde{\Hbf}\in\mathbb{C}^{M\times 3N}$ is the matrix representation of the convolution operator induced by the third row of the dyadic Green's function \eqref{eq:dyadicGreen}.
Define
\begin{equation}
\widetilde{\Abf}:=\Ibf - (k^2\Ibf + \Dbf)\left(\Ibf_3\otimes (\widetilde{\Gbf}\diag{\widetilde{\fbf}})\right).
\label{eq:3DA_def}
\end{equation}
Similar to the scalar field case, we can see that the measurement vector $\widetilde{\ybf}$ is nonlinear in the unknown $\widetilde{\fbf}$, because $\widetilde{\ubf}$ depends on $\widetilde{\fbf}$ according to $\widetilde{\ubf}=\widetilde{\Abf}^{-1}\widetilde{\ubf}^{\text{in}}$, which follows from \eqref{eq:discrete_constrain}.

Next, we define the discrete gradient-divergence operator $\grad\divergence$, for which the matrix representation is $\Dbf$ in \eqref{eq:discrete_constrain}. For a scalar function $f$ of the spatial location $\rbf_{l,s,t}$, denote $f(\rbf_{l,s,t})$ by $f_{l,s,t}$. Following the finite difference rule,
\begin{align*}
\frac{\partial^2}{\partial x^2}f_{l,s,t}&:=\frac{f_{l-1,s,t} - 2f_{l,s,t} + f_{l+1,s,t}}{\delta^2},\\
\frac{\partial^2}{\partial x\partial y}f_{l,s,t}&:=\frac{f_{l-1,s-1,t}-f_{l-1,s+1,t}}{4\delta^2}\\
&+\frac{f_{l+1,s+1,t} - f_{l+1,s-1,t}}{4\delta^2}.
\end{align*} 
The definitions of $\frac{\partial^2}{\partial x\partial z}$ and $\frac{\partial^2}{\partial y \partial z}$ are similar to that of $\frac{\partial^2}{\partial x \partial y}$. Moreover, for a vector $\vec{E}_{m,n,p}\in\mathbb{C}^3$, we use $E_{m,n,p}^{(i)}$ to denote its $i^\text{th}$ coordinate.
Then we have that the first coordinate of $\vec{E}_{m,n,p}$ is 
\begin{align*}
&E_{m,n,p}^{(1)}=E_{m,n,p}^{\text{in},(1)} + k^2 B^{(1)}_{m,n,p}\\
& + \frac{\partial}{\partial x_1}\left(\frac{\partial}{\partial x_1}B^{(1)}_{m,n,p}+ \frac{\partial}{\partial x_2}B^{(2)}_{m,n,p}+ \frac{\partial}{\partial x_3} B^{(3)}_{m,n,p}\right)\\
&=E_{m,n,p}^{\text{in},(1)} + k^2 B^{(1)}_{m,n,p}+ \frac{B_{m+1,n,p}^{(1)}-2B_{m,n,p}^{(1)}+B_{m-1,n,p}^{(1)}}{\delta^2}\\
&+\frac{B^{(2)}_{m-1,n-1,p}-B^{(2)}_{m-1,n+1,p}-B^{(2)}_{m+1,n-1,p}+B^{(2)}_{m+1,n+1,p}}{4\delta^2}\\
&+\frac{B^{(3)}_{m-1,n,p-1}-B^{(3)}_{m-1,n,p+1}-B^{(3)}_{m+1,n,p-1}+B^{(3)}_{m+1,n,p+1}}{4\delta^2}.
\end{align*}
The second and third coordinates of $\vec{E}_{m,n,p}$ can be obtained in a similar way.

\subsubsection{CISOR for 3D}
\label{app:3DCISOR}

The data-fidelity term $\Dcal(\cdot)$ in the objective function \eqref{eq:objective_func} is now defined as $\Dcal(\widetilde{\fbf}):=\frac{1}{2}\|\widetilde{y} - \widetilde{\Zcal}(\widetilde{\fbf})\|^2$, where $\widetilde{Z}(\widetilde{\fbf}):=\widetilde{\Hbf} \left(\Ibf_3\otimes\diag{\widetilde{\fbf}}\right)\widetilde{\ubf}$.
Let $\widetilde{\wbf}=\widetilde{\Zcal}(\widetilde{\fbf})-\widetilde{\ybf}$, and
\begin{equation}
\gbf = \diag{\widetilde{\ubf}}^\mathsf{H}\left( \widetilde{\Hbf}^\mathsf{H}\wbf + (\Ibf_3\otimes\widetilde{\Gbf}^\mathsf{H}) (k^2\Ibf+\Dbf^\mathsf{H})\widetilde{\vbf} \right), 
\label{eq:3Dg_def}
\end{equation}
where $\widetilde{\ubf}$ and $\widetilde{\vbf}$ are obtained from the linear systems
\begin{equation}
\widetilde{\Abf}\widetilde{\ubf} = \widetilde{\ubf}^\text{in},\quad\text{and}\quad \widetilde{\Abf}^\mathsf{H}\widetilde{\vbf}=(\Ibf_3\otimes\diag{\widetilde{\fbf}})\widetilde{\Hbf}^\mathsf{H}\widetilde{\wbf},
\label{eq:3Du_and_v}
\end{equation}
where $\widetilde{\Abf}$ is defined in \eqref{eq:3DA_def}. Then the gradient of $\Dcal$ can be written as
\begin{equation}
\grad\Dcal(\fbf) = \textsf{Re}\left\{\sum_{i=1}^3 \gbf^{(i)}\right\}
\label{eq:3Dgradient}
\end{equation}
with  $\mathbf{g}^{(i)}=(g_{(i-1)N+1},\ldots,g_{iN})\in\mathbb{C}^N$ for $i=1,2,3$.

\subsection{Proofs for Theoretical Results}
\label{app:proofs}

\subsubsection{Proof for Proposition \ref{prop:gradient}}
\label{app:prop_grad_proof}
The gradient of $\Dcal(\cdot)$ defined in \eqref{eq:D_def} is $\textsf{Re}\left\{\mathbf{J}_\Zcal^\H\wbf\right\}$, where $\mathbf{J}_\Zcal$ is the Jocobian matrix of $\Zcal(\cdot)$ \ref{eq:Z_def}, which is defined as
\begin{equation*}
\mathbf{J}_{\Zcal} = \left[
\begin{array}{ccc}
\frac{\partial \Zcal_1}{\partial f_1}\quad &\ldots \quad &\frac{\partial \Zcal_1}{\partial f_N}\\
\vdots \quad &\ddots \quad &\vdots\\
\frac{\partial \Zcal_M}{\partial f_1}\quad &\ldots \quad&\frac{\partial \Zcal_M}{\partial f_N} 
\end{array}\right].
\end{equation*}
Recall that $\Abf =( \Ibf - \Gbf\diag{\fbf})$ and $\ubf = \Abf^{-1}\ubf_{\text{in}}$, hence both $\Abf$ and $\ubf$ are functions of $\fbf$. We omit such dependencies in our notation for brevity.
Following the chain rule of differentiation, 
\begin{align*}
\frac{\partial \Zcal_m}{\partial f_n} &= \frac{\partial}{\partial f_n} \sum_{i=1}^{N} H_{m,i} f_i u_i\\
& = H_{m,n} u_n + \sum_{i=1}^N \left[\frac{\partial u_i}{\partial f_n}\right] H_{m,i}f_i.
\end{align*}
Using the definition $\wbf = \Zcal(\fbf) - \ybf$ and summing over $m=1,...,M$,
\begin{align}
&\left[\grad\Dcal(\fbf)\right]_n = \sum_{m=1}^{M} \conj{\left[\frac{\partial \Zcal_m}{\partial f_n}\right]} w_m\nonumber\\
& = \conj{u_n}\sum_{m=1}^M \conj{H_{m,n}} w_m + \sum_{i=1}^{N} \conj{\left[\frac{\partial u_i}{\partial f_n}\right]} f_i \sum_{m=1}^M \conj{H_{m,i}} w_m\nonumber\\
& = \conj{u_n}\left[\Hbf^\H \wbf \right]_n + \sum_{i=1}^{N}  \conj{\left[\frac{\partial u_i}{\partial f_n}\right]} f_i \left[\Hbf^\H \wbf \right]_i,
\label{eq:grad_two_terms}
\end{align}
where $\conj{a}$ denotes the complex conjugate of $a\in\mathbb{C}$.
Label the two terms in \eqref{eq:grad_two_terms} as $T_1$ and $T_2$, then 
\begin{equation}
T_1 = \left[\diag{\ubf}^\H\Hbf^\H\wbf\right]_n,
\label{eq:T1}
\end{equation}
and
\begin{align}
T_2 &\overset{(a)}{=} (\ubf^{\text{in}})^\H \left[ \frac{\partial \Abf^{-1}}{\partial f_n}\right]^\H \diag{\fbf} \Hbf^\H\wbf\nonumber\\
&\overset{(b)}{=}-(\ubf^{\text{in}})^\H \Abf^{-\H} \left[\frac{\partial \Abf}{\partial f_n}\right]^\H \Abf^{-\H}\diag{\fbf}\Hbf^\H\wbf\nonumber\\
&\overset{(c)}{=} -\ubf^\H(\fbf) \left[\frac{\partial \Abf}{\partial f_n}\right]^\H \vbf \overset{(d)}{=} [\diag{\ubf}^\H\Gbf^\H\vbf]_n.\label{eq:T2}
\end{align}
In the above, step $(a)$ holds by plugging in $u_i= \left[ \Abf^{-1}\ubf_{\text{in}}\right]_i$. Step $(b)$ uses the identity 
\begin{equation}
\frac{\partial \Abf^{-1}}{\partial f_n}=-\Abf^{-1}\frac{\partial \Abf}{\partial f_n}\Abf^{-1},
\label{eq:derivInverseMatrix}
\end{equation}
which follows by differentiating both sides of $\Abf\Abf^{-1}=\Ibf$, 
\begin{equation*}
\frac{\partial \Abf}{\partial f_n}\Abf^{-1} + \Abf\frac{\partial \Abf^{-1}}{\partial f_n}=0.
\end{equation*}
From step $(b)$ to step $(c)$, we used the fact that $\ubf=\Abf^{-1}\ubf^{\text{in}}$ and defined
$\vbf :=  \Abf^{-\H}\diag{\fbf}\Hbf^\H\wbf$, which matches \eqref{eq:u_and_v}. Finally, step $(d)$ follows by plugging in $\Abf = \Ibf - \Gbf\diag{\fbf}$. Combining \eqref{eq:grad_two_terms}, \eqref{eq:T1}, and \eqref{eq:T2}, we have obtained the expression in \eqref{eq:gradient}. 

Note that the extension to the 3D vectorial field case \eqref{eq:3Dgradient} is straightforward. We highlight the differences from the scalar field case in the following.
Let $\widetilde{\Hbf}=[\widetilde{\Hbf}^{(1)} \vert \widetilde{\Hbf}^{(2)} \vert \widetilde{\Hbf}^{(3)}]$, where $\widetilde{\Hbf}^{(i)}\in\mathbb{C}^{M\times N}$ for $i=1,2,3$.
Following the chain rule of differentiation, we have
\begin{equation*}
\frac{\partial \widetilde{\Zcal}_m}{\partial \widetilde{f}_n} = \sum_{k=1}^3 \widetilde{H}^{(k)}_{m,n} \widetilde{u}^{(k)}_n + \sum_{k=1}^{3}\sum_{i=1}^N \left[\frac{\partial \widetilde{u}^{(k)}_i}{\partial \widetilde{f}_n}\right] \widetilde{H}^{(k)}_{m,i}\widetilde{f}_i.
\end{equation*}
Using the definition $\widetilde{\wbf}$ and summing over $m=1,...,M$,
\begin{align}
&\left[\grad\Dcal(\widetilde{\fbf})\right]_n = \sum_{m=1}^{M} \conj{\left[\frac{\partial \widetilde{\Zcal}_m}{\partial \widetilde{f}_n}\right]} \widetilde{w}_m = \sum_{k=1}^3 \conj{\widetilde{u}^{(k)}_n}\left[(\widetilde{\Hbf}^{(k)})^\mathsf{H}\widetilde{\wbf} \right]_n \nonumber\\
&+ \sum_{k=1}^3\sum_{i=1}^{N}  \conj{\left[\frac{\partial \widetilde{u}^{(k)}_i}{\partial \widetilde{f}_n}\right]} \widetilde{f}_i \left[(\widetilde{\Hbf}^{(k)})^\mathsf{H} \widetilde{\wbf} \right]_i.
\label{eq:3Dgrad_two_terms}
\end{align}
Label the two terms in \eqref{eq:3Dgrad_two_terms} as $T_1$ and $T_2$, then 
\begin{align}
T_1 &= \sum_{k=1}^{3}\left[\diag{\widetilde{\ubf}^{(k)}}^\mathsf{H}(\widetilde{\Hbf}^{(k)})^\mathsf{H}\widetilde{\wbf}\right]_n,\label{eq:3DT1}\\
T_2 &\overset{(a)}{=} \sum_{k=1}^3 \left(\left[ \frac{\partial \Abf^{-1}}{\partial f_n} \ubf^{\text{in}}\right]^{(k)}\right)^\mathsf{H} \diag{\fbf} (\Hbf^{(k)})^\mathsf{H}\wbf\nonumber\\
&\overset{(b)}{=} \sum_{k=1}^3 \left(\left[ -\widetilde{\Abf}^{-1} \frac{\partial \widetilde{\Abf}}{\partial \widetilde{f}_n} \widetilde{\ubf}\right]^{(k)}\right)^{\mathsf{H}}\diag{\widetilde{\fbf}}(\widetilde{\Hbf}^{(k)})^\mathsf{H}\widetilde{\wbf}\nonumber\\
&\overset{(c)}{=} \sum_{k=1}^{3} \left[-\widetilde{\ubf}^\mathsf{H} \left[\frac{\partial \widetilde{\Abf}}{\partial \widetilde{f}_n}\right]^\mathsf{H} \widetilde{\vbf}\right]^{(k)}\nonumber\\
&\overset{(d)}{=} \sum_{k=1}^3 \left[ \diag{\widetilde{\ubf}^{(k)}}^{\mathsf{H}} \widetilde{\Gbf}^\mathsf{H} [(k^2\Ibf+\Dbf^\mathsf{H})\widetilde{\vbf}]^{(k)}\right]_n.\label{eq:3DT2}
\end{align}
In the above, step $(a)$ follows from $\widetilde{u}^{(k)}_i = \left[ \widetilde{\Abf}^{-1}\widetilde{\ubf}^{\text{in}}\right]^{(k)}_i$. Step $(b)$ follows from \eqref{eq:derivInverseMatrix}.
In step $(c)$, we defined
$\widetilde{\vbf} :=  \widetilde{\Abf}^{-\mathsf{H}}(\Ibf_3\otimes\diag{\widetilde{\fbf}})\widetilde{\Hbf}^\mathsf{H}\widetilde{\wbf}$, which matches \eqref{eq:3Du_and_v}. 
Finally, step $(d)$ follows by plugging in the definition of $\widetilde{\Abf}$ in \eqref{eq:3DA_def}. 
Combining \eqref{eq:3Dgrad_two_terms}, \eqref{eq:3DT1}, and \eqref{eq:3DT2}, we have obtained the expression in \eqref{eq:3Dgradient}. 


\subsubsection{Proof for Proposition \ref{prop:lip_grad}}
\label{app:prop_lipschitz_proof}

For a vector $\abf$ that is a function of $\sbf$, $\abf_i$ denotes the value of $\abf$ evaluated at $\sbf_i$.  Using this notation, we have
\begin{align*}
&\|\grad\Dcal(\sbf_1) - \grad\Dcal(\sbf_2)\| \leq \|\diag{\ubf_1}^\H\Hbf^\H\wbf_1 - \diag{\ubf_2}^\H\Hbf^\H\wbf_2\|\\
&\qquad+  \|\diag{\ubf_1}^\H\Gbf^\H\vbf_1 - \diag{\ubf_2}^\H\Gbf^\H\vbf_2\|.
\end{align*}
Label the two terms on the RHS as $T_1$ and $T_2$.
We will prove $T_1\leq L_1\|\sbf_1 - \sbf_2\|$ for some $L_1\in(0,\infty)$, and $T_2\leq L_2\|\sbf_1 - \sbf_2\|$ can be proved in a similar way for some $L_2\in(0,\infty)$. The result \eqref{eq:Lips_constant} is then obtained by letting $L=L_1 + L_2$.
\begin{align*}
T_1 &\leq \|\diag{\ubf_1}^\H\Hbf^\H\wbf_1 - \diag{\ubf_2}^\H\Hbf^\H\wbf_1\|\\
&+ \|\diag{\ubf_2}^\H\Hbf^\H\wbf_1 - \diag{\ubf_2}^\H\Hbf^\H\wbf_2\|\\
&\leq \|\ubf_1 - \ubf_2\|\|\Hbf\|_{\text{op}}\|\wbf_1\| + \|\Abf_2^{-1}\|\|\ubf^{\text{in}}\|\|\Hbf\|_{\text{op}}\|\wbf_1 - \wbf_2\|,
\end{align*}
where $\|\cdot\|_{\text{op}}$ denotes the operator norm and the last inequality uses the fact that $\|\diag{\mathbf{d}}\|_{\text{op}}=\max_{n\in [N]}|d_n|\leq \|\mathbf{d}\|$. We now bound $\|\ubf_1-\ubf_2\|$ and $\|\wbf_1 - \wbf_2\|$.
\begin{align*}
&\|\ubf_1 - \ubf_2\| = \|\Abf_1^{-1}\ubf^{\text{in}} - \Abf_2^{-1}\ubf^{\text{in}}\|\leq \|\Abf_1^{-1}-\Abf_2^{-1}\|_\text{op}\|\ubf^\text{in}\|\\
&\quad= \|\Abf_1^{-1}(\Abf_2 -\Abf_1)\Abf_2^{-1}\|_{\text{op}}\|\ubf^\text{in}\|\\
&\quad \leq \|\Abf_1^{-1}\|_{\text{op}} \|\Gbf\|_{\text{op}} \|\sbf_1-\sbf_2\| \|\Abf_2^{-1}\|_{\text{op}} \|\ubf^{\text{in}}\|,\\
&\|\wbf_1 - \wbf_2\| \leq \|\Hbf\diag{\sbf_1}\ubf_1 - \Hbf\diag{\sbf_1}\ubf_2\|\\
&\quad+ \|\Hbf\diag{\sbf_1}\ubf_2 - \Hbf\diag{\sbf_2}\ubf_2\|\\
&\quad\leq \|\Hbf\|_\text{op}\|\sbf_1\|\|\ubf_1 - \ubf_2\| + \|\Hbf\|_\text{op}\|\sbf_1 - \sbf_2\|\|\Abf_2^{-1}\|_\text{op}\|\ubf^\text{in}\|.
\end{align*}
Then the result $T_1\leq L_1\|\sbf_1 - \sbf_2\|$ follows by noticing that $\|\sbf_1\|$, $\|\ubf^{\text{in}}\|$, $\|\Gbf\|_\text{op}$, $\|\Hbf\|_{\text{op}}$, and $\|\Abf^{-1}_{i}\|_\text{op}$ for $i=1,2$ are bounded, and the fact that $\|\wbf_1\|\leq\|\ybf\| + \|\Hbf\|_\text{op}\|\sbf_1\|\|\Abf_1^{-1}\|_\text{op}\|\ubf^{\text{in}}\|< \infty$.

The Lipschitz property of the gradient \eqref{eq:3Dgradient} in the 3D case can be proved in a similar way.

\subsubsection{Proof for Proposition \ref{prop:converge}}
\label{app:prop_converge_proof}

First, we show that the Lipschitz gradient condition \eqref{eq:Lips_constant} implies that for all $\xbf,\ybf\in\mathcal{U}$,
\begin{equation}
\Dcal(\ybf) - \Dcal(\xbf) - \langle \grad \Dcal(\xbf), \ybf - \xbf \rangle \geq -\frac{L}{2}\|\xbf - \ybf\|^2.
\label{eq:LipsGradProperty}
\end{equation}
Define a function $h:\mathbb{R}\to\mathbb{R}$ as $h(\lambda):=\Dcal(\xbf+\lambda(\ybf-\xbf))$, then $h'(\lambda)=\langle \grad \Dcal(\xbf+\lambda(\ybf-\xbf)),\ybf-\xbf\rangle$. Notice that $h(1)=\Dcal(\ybf)$ and $h(0)=\Dcal(\xbf)$. Using the equality $h(1)=h(0)+\int_0^1 h'(\lambda)d\lambda$, we have that
\begin{align*}
\Dcal(\ybf) &= \Dcal(\xbf) + \int_0^1 \langle \grad \Dcal(\xbf+\lambda(\ybf-\xbf)),\ybf-\xbf\rangle d\lambda\\
&= \Dcal(\xbf) + \int_0^1 \langle \grad\Dcal(\xbf), \ybf-\xbf\rangle d\lambda\\
&+\int_0^1 \langle \grad \Dcal(\xbf+\lambda(\ybf-\xbf)) - \grad \Dcal(\xbf),\ybf-\xbf\rangle d\lambda\\
&\overset{(a)}{\geq} \Dcal(\xbf) - \int_0^1 \lambda L\|\ybf-\xbf\|^2 d\lambda + \langle \grad\Dcal(\xbf), \ybf-\xbf\rangle\\
&=\Dcal(\xbf) - \frac{L}{2}\|\ybf-\xbf\|^2 + \langle \grad\Dcal(\xbf), \ybf-\xbf\rangle,
\end{align*}
where step $(a)$ uses Cauchy-Schwarz inequality and the Lipschitz gradient condition \eqref{eq:Lips_constant}.

Next, 
by \eqref{eq:algo1}, we have that for all $\xbf\in \mathcal{U}$, $t\geq 0$,
\begin{equation}
\Rcal(\xbf) \geq \Rcal(\fbf_{t}) + \langle \frac{\sbf_{t}-\fbf_{t}}{\gamma} - \grad\Dcal(\sbf_{t}), \xbf - \fbf_{t} \rangle.
\label{eq:subgradConvexFunction}
\end{equation}
Let $\ybf = \fbf_{k}$, $\xbf = \fbf_{k+1}$ in \eqref{eq:LipsGradProperty} and $\xbf = \fbf_k$, $t=k+1$ in \eqref{eq:subgradConvexFunction}. Adding \eqref{eq:LipsGradProperty} and \eqref{eq:subgradConvexFunction}, we have
\begin{align}
&\mathcal{F}(\fbf_{k+1}) - \mathcal{F}(\fbf_{k}) \leq \langle \grad\Dcal(\fbf_{k+1}) - \grad\Dcal(\sbf_{k+1}),\fbf_{k+1}-\fbf_k \rangle\nonumber\\
& + \frac{1}{\gamma}\langle \sbf_{k+1} - \fbf_{k+1},\fbf_{k+1} - \fbf_{k} \rangle + \frac{L}{2}\|\fbf_{k+1} - \fbf_{k}\|^2 \nonumber\\
& \overset{(a)}{\leq} \frac{L}{2}\|\sbf_{k+1}-\fbf_{k+1}\|^2+\frac{L}{2}\|\fbf_{k+1} - \fbf_k\|^2 + \frac{1}{2\gamma}\|\sbf_{k+1}-\fbf_k\|^2\nonumber\\
& -\frac{1}{2\gamma}\|\sbf_{k+1}-\fbf_{k+1}\|^2 - \frac{1}{2\gamma}\|\fbf_{k+1}-\fbf_k\|^2 + \frac{L}{2}\|\fbf_{k+1}-\fbf_k\|^2\nonumber\\
&\overset{(b)}{\leq} \left(\frac{1}{2\gamma} - L\right)\|\fbf_k-\fbf_{k-1}\|^2 - \left(\frac{1}{2\gamma} - L\right)\|\fbf_{k+1} - \fbf_{k}\|^2\nonumber\\
&- \left( \frac{1}{2\gamma}-\frac{L}{2} \right) \|\sbf_{k+1}-\fbf_{k+1}\|^2. \nonumber
\end{align}
In the above, step $(a)$ uses Cauchy-Schwarz, Proposition \ref{prop:lip_grad}, as well as the fact that $2ab\leq a^2 + b^2$ and $2\langle \mathbf{a} - \mathbf{b} , \mathbf{b} - \mathbf{c} \rangle=\|\mathbf{a}-\mathbf{c}\|^2 - \|\mathbf{a} - \mathbf{b}\|^2 - \|\mathbf{b} - \mathbf{c}\|^2$. Step $(b)$ uses the condition in the proposition statement that $\gamma\leq \frac{1-\alpha^2}{2L}$ and \eqref{eq:algo3}, which implies $\|\sbf_{k+1}-\fbf_k\|\leq \alpha\frac{t_{k}-1}{t_{k+1}}\|\fbf_k - \fbf_{k-1}\|$, where we notice that $\frac{t_k-1}{t_{k+1}}\leq 1$ by \eqref{eq:algo2}, and $\alpha<1$ by our assumption.
Summing both sides from $k=0$ to $K$:
\begin{align*}
&\left( \frac{1}{2\gamma}-\frac{L}{2} \right)\sum_{k=0}^{K-1}\|\sbf_{k+1}-\fbf_{k+1}\|^2 \leq \mathcal{F}(\fbf_0) - \mathcal{F}(\fbf_{K})\nonumber\\
& + \left(\frac{1}{2\gamma} - L\right) \left(\|\fbf_0-\fbf_{-1}\|^2 - \|\fbf_{K} - \fbf_{K-1}\|^2\right)\leq\mathcal{F}(\fbf_0) - \mathcal{F}^*,
\end{align*}
where $\mathcal{F}^*$ is the global minimum. The last step follows by letting $\fbf_{-1}=\fbf_0$, which satisfies \eqref{eq:algo3} for the initialization $\sbf_1=\fbf_0$, and the fact that $\mathcal{F}^*\leq \mathcal{F}(\fbf_K)$. 

Recall the definition of the gradient mapping $\mathcal{G}_{\gamma}$ in \eqref{eq:gradient_mapping_def}, we have $\mathcal{G}_{\gamma}(\sbf_{k})=\frac{\sbf_k - \fbf_k}{\gamma}$. Therefore,
\begin{equation}
\sum_{k=1}^{K}\|\mathcal{G}_\gamma(\sbf_k)\|^2\leq \frac{2L\left(\mathcal{F}(\fbf_0) - \mathcal{F}^*\right)}{\gamma L(1 - \gamma L)},
\label{eq:partial_sum}
\end{equation}
which implies that $\lim_{K\to\infty}\sum_{k=1}^{K}\|\mathcal{G}_\gamma(\sbf_k)\|^2<\infty$, hence
\begin{equation}
\lim_{k\rightarrow\infty} \|\mathcal{G}_\gamma(\sbf_k)\|=0.
\label{eq:gradient_mapping_converge_s}
\end{equation} 
Note that \eqref{eq:gradient_mapping_converge_s} establishes that the gradient mapping acting on the sequence $\{\sbf_k\}_{k\geq 0}$ generated from relaxed FISTA converges to 0. In the following, we show that the gradient mapping acting on $\{\fbf_k\}_{k\geq 0}$ converges to 0 as well, which is the desired result stated in \eqref{eq:gradient_mapping_converge}.
By \eqref{eq:algo1} and \eqref{eq:gradient_mapping_def}, we have $\mathcal{G}_{\gamma}(\sbf_k)=\frac{1}{\gamma}\left(\sbf_k-\fbf_k\right)$. Therefore, \eqref{eq:gradient_mapping_converge_s} implies that $\lim_{k\to\infty}\|\sbf_k-\fbf_k\|=0$. Moreover, 
\begin{align}
&\|\mathcal{G}_{\gamma}(\fbf_k)-\mathcal{G}_{\gamma}(\sbf_k)\| \overset{(a)}{\leq} \frac{1}{\gamma}\|\sbf_k-\fbf_k\|\nonumber\\
&+ \frac{1}{\gamma}\|\prox_{\gamma\Rcal}(\fbf_k - \gamma\grad\Dcal(\fbf_k)) - \prox_{\gamma\Rcal}(\sbf_k - \gamma\grad\Dcal(\sbf_k))\|\nonumber\\
&\overset{(b)}{\leq} \frac{1}{\gamma}\|\sbf_k-\fbf_k\| + \frac{1}{\gamma}\|\fbf_k - \gamma\grad\Dcal(\fbf_k) - \left(\sbf_k - \gamma\grad\Dcal(\sbf_k)\right)\|\nonumber\\
&\overset{(c)}{\leq}\frac{1}{\gamma}\|\sbf_k-\fbf_k\| + \frac{1}{\gamma}\|\sbf_k-\fbf_k\| +  L\|\sbf_k-\fbf_k\|,
\label{eq:Gfs}
\end{align}
where step $(a)$ follows by \eqref{eq:gradient_mapping_def} and the triangle inequality, step $(b)$ follows by the well-known property that the proximal operators for convex functions are Lipschitz-1, and step $(c)$ follows by the triangle inequality and the result that $\grad\Dcal$ is Lipschitz-$L$ as we established in Proposition \ref{prop:lip_grad}.
Taking the limit as $k\to\infty$, we have $\lim_{k\to\infty}\|\mathcal{G}_{\gamma}(\fbf_k)-\mathcal{G}_{\gamma}(\sbf_k)\|=0$, hence, $\lim_{k\to\infty}\mathcal{G}_{\gamma}(\fbf_k)=\lim_{k\to\infty}\mathcal{G}_{\gamma}(\sbf_k)=0$.

Moreover, by \eqref{eq:Gfs},
\begin{equation*}
\|\mathcal{G}_{\gamma}(\fbf_k)\|-\|\mathcal{G}_{\gamma}(\sbf_k)\| \leq \|\mathcal{G}_{\gamma}(\fbf_k)-\mathcal{G}_{\gamma}(\sbf_k)\| \leq (2 + \gamma L) \|\mathcal{G}_{\gamma}(\sbf_k)\|,
\end{equation*}
hence, $\|\mathcal{G}_{\gamma}(\fbf_k)\|^2\leq  (3+\gamma L)^2\|\mathcal{G}_{\gamma}(\sbf_k)\|^2$ for all $k\geq 0$. Then by \eqref{eq:partial_sum},
\begin{equation*}
\sum_{k=1}^{K}\|\mathcal{G}_\gamma(\fbf_k)\|^2\leq (3+\gamma L)^2 \frac{2L\left(\mathcal{F}(\fbf_0) - \mathcal{F}^*\right)}{\gamma L(1 - \gamma L)},
\end{equation*}
which implies
\begin{equation*}
K\min_{k\in\{1,\ldots,K\}}\|\mathcal{G}_\gamma(\fbf_k)\|^2 \leq (3+\gamma L)^2\frac{2L\left(\mathcal{F}(\fbf_0) - \mathcal{F}^*\right)}{\gamma L(1 - \gamma L)},
\end{equation*}
hence
\begin{equation*}
\min_{k\in\{1,\ldots,K\}}\|\mathcal{G}_\gamma(\fbf_k)\|^2\leq (3+\gamma L)^2\frac{2L\left(\mathcal{F}(\fbf_0) - \mathcal{F}^*\right)}{K\gamma L(1 - \gamma L)}.
\end{equation*}


\bibliographystyle{IEEEtran}

\begin{thebibliography}{10}
\providecommand{\url}[1]{#1}
\csname url@samestyle\endcsname
\providecommand{\newblock}{\relax}
\providecommand{\bibinfo}[2]{#2}
\providecommand{\BIBentrySTDinterwordspacing}{\spaceskip=0pt\relax}
\providecommand{\BIBentryALTinterwordstretchfactor}{4}
\providecommand{\BIBentryALTinterwordspacing}{\spaceskip=\fontdimen2\font plus
\BIBentryALTinterwordstretchfactor\fontdimen3\font minus
  \fontdimen4\font\relax}
\providecommand{\BIBforeignlanguage}[2]{{%
\expandafter\ifx\csname l@#1\endcsname\relax
\typeout{** WARNING: IEEEtran.bst: No hyphenation pattern has been}%
\typeout{** loaded for the language `#1'. Using the pattern for}%
\typeout{** the default language instead.}%
\else
\language=\csname l@#1\endcsname
\fi
#2}}
\providecommand{\BIBdecl}{\relax}
\BIBdecl

\bibitem{Born.Wolf2003}
M.~Born and E.~Wolf, \emph{Principles of Optics}, 7th~ed.\hskip 1em plus 0.5em
  minus 0.4em\relax Cambridge Univ. Press, 2003, ch. Scattering from
  inhomogeneous media, pp. 695--734.

\bibitem{Devaney1981}
A.~J. Devaney, ``Inverse-scattering theory within the {R}ytov approximation,''
  \emph{Opt. Lett.}, vol.~6, no.~8, pp. 374--376, August 1981.

\bibitem{Bronstein.etal2002}
M.~M. Bronstein, A.~M. Bronstein, M.~Zibulevsky, and H.~Azhari,
  ``Reconstruction in diffraction ultrasound tomography using nonuniform
  {FFT},'' \emph{IEEE Trans. Med. Imag.}, vol.~21, no.~11, pp. 1395--1401,
  November 2002.

\bibitem{Lim.etal2015}
J.~W. Lim, K.~R. Lee, K.~H. Jin, S.~Shin, S.~E. Lee, Y.~K. Park, and J.~C. Ye,
  ``Comparative study of iterative reconstruction algorithms for missing cone
  problems in optical diffraction tomography,'' \emph{Opt. Express}, vol.~23,
  no.~13, pp. 16\,933--16\,948, June 2015.

\bibitem{Sung.Dasari2011}
Y.~Sung and R.~R. Dasari, ``Deterministic regularization of three-dimensional
  optical diffraction tomography,'' \emph{J. Opt. Soc. Am. A}, vol.~28, no.~8,
  pp. 1554--1561, August 2011.

\bibitem{Lauer2002}
V.~Lauer, ``New approach to optical diffraction tomography yielding a vector
  equation of diffraction tomography and a novel tomographic microscope,''
  \emph{J. Microsc.}, vol. 205, no.~2, pp. 165--176, 2002.

\bibitem{Sung.etal2009}
Y.~Sung, W.~Choi, C.~Fang-Yen, K.~Badizadegan, R.~R. Dasari, and M.~S. Feld,
  ``Optical diffraction tomography for high resolution live cell imaging,''
  \emph{Opt. Express}, vol.~17, no.~1, pp. 266--277, December 2009.

\bibitem{Kim.etal2014a}
T.~Kim, R.~Zhou, M.~Mir, S.~Babacan, P.~Carney, L.~Goddard, and G.~Popescu,
  ``Supplementary information: White-light diffraction tomography of unlabelled
  live cells,'' \emph{Nat. Photonics}, vol.~8, pp. 256--263, March 2014.

\bibitem{Sharpe.etal2002}
J.~Sharpe, U.~Ahlgren, P.~Perry, B.~Hill, A.~Ross, J.~Hecksher-S{\o}rensen,
  R.~Baldock, and D.~Davidson, ``Optical projection tomography as a tool for
  {3D} microscopy and gene expression studies,'' \emph{Science}, vol. 296, no.
  5567, pp. 541--545, April 2002.

\bibitem{Choi.etal2007a}
W.~Choi, C.~Fang-Yen, K.~Badizadegan, S.~Oh, N.~Lue, R.~R. Dasari, and M.~S.
  Feld, ``Tomographic phase microscopy: Supplementary material,'' \emph{Nat.
  Methods}, vol.~4, no.~9, pp. 1--18, September 2007.

\bibitem{Ralston.etal2006}
T.~S. Ralston, D.~L. Marks, P.~S. Carney, and S.~A. Boppart, ``Inverse
  scattering for optical coherence tomography,'' \emph{J. Opt. Soc. Am. A},
  vol.~23, no.~5, pp. 1027--1037, May 2006.

\bibitem{Davis.etal2007}
B.~J. Davis, S.~C. Schlachter, D.~L. Marks, T.~S. Ralston, S.~A. Boppart, and
  P.~S. Carney, ``Nonparaxial vector-field modeling of optical coherence
  tomography and interferometric synthetic aperture microscopy,'' \emph{J. Opt.
  Soc. Am. A}, vol.~24, no.~9, pp. 2527--2542, September 2007.

\bibitem{Brady.etal2009}
D.~J. Brady, K.~Choi, D.~L. Marks, R.~Horisaki, and S.~Lim, ``Compressive
  holography,'' \emph{Opt. Express}, vol.~17, no.~15, pp. 13\,040--13\,049,
  2009.

\bibitem{Tian.etal2010}
L.~Tian, N.~Loomis, J.~A. Dominguez-Caballero, and G.~Barbastathis,
  ``Quantitative measurement of size and three-dimensional position of
  fast-moving bubbles in air–water mixture flows using digital holography,''
  \emph{Appl. Opt.}, vol.~49, no.~9, pp. 1549--1554, March 2010.

\bibitem{Chen.etal2015a}
W.~Chen, L.~Tian, S.~Rehman, Z.~Zhang, H.~P. Lee, and G.~Barbastathis,
  ``Empirical concentration bounds for compressive holographic bubble imaging
  based on a {M}ie scattering model,'' \emph{Opt. E}, vol.~23, no.~4, p.
  February, 2015.

\bibitem{Jol2009}
H.~M. Jol, Ed., \emph{Ground Penetrating Radar: Theory and Applications}.\hskip
  1em plus 0.5em minus 0.4em\relax Amsterdam: Elsevier, 2009.

\bibitem{Leigsnering.etal2014a}
M.~Leigsnering, F.~Ahmad, M.~Amin, and A.~Zoubir, ``Multipath exploitation in
  through-the-wall radar imaging using sparse reconstruction,'' \emph{IEEE
  Trans. Aerosp. Electron. Syst.}, vol.~50, no.~2, pp. 920--939, April 2014.

\bibitem{Liu.etal2016a}
D.~Liu, U.~S. Kamilov, and P.~T. Boufounos, ``Compressive tomographic radar
  imaging with total variation regularization,'' in \emph{Proc. {IEEE} 4th
  International Workshop on Compressed Sensing Theory and its Applications to
  Radar, Sonar, and Remote Sensing ({CoSeRa} 2016)}, Aachen, Germany, September
  19-22, 2016, pp. 120--123.

\bibitem{Boyd.Vandenberghe2004}
S.~Boyd and L.~Vandenberghe, \emph{Convex Optimization}.\hskip 1em plus 0.5em
  minus 0.4em\relax Cambridge Univ. Press, 2004.

\bibitem{Nocedal.Wright2006}
J.~Nocedal and S.~J. Wright, \emph{Numerical Optimization}, 2nd~ed.\hskip 1em
  plus 0.5em minus 0.4em\relax Springer, 2006.

\bibitem{Bioucas-Dias.Figueiredo2007}
J.~M. Bioucas-Dias and M.~A.~T. Figueiredo, ``A new {T}w{IST}: {T}wo-step
  iterative shrinkage/thresholding algorithms for image restoration,''
  \emph{IEEE Trans. Image Process.}, vol.~16, no.~12, pp. 2992--3004, December
  2007.

\bibitem{Beck.Teboulle2009}
A.~Beck and M.~Teboulle, ``A fast iterative shrinkage-thresholding algorithm
  for linear inverse problems,'' \emph{SIAM J. Imaging Sciences}, vol.~2,
  no.~1, pp. 183--202, 2009.

\bibitem{Chen.Stamnes1998}
B.~Chen and J.~J. Stamnes, ``Validity of diffraction tomography based on the
  first born and the first rytov approximations,'' \emph{Appl. Opt.}, vol.~37,
  no.~14, pp. 2996--3006, May 1998.

\bibitem{Ntziachristos2010}
V.~Ntziachristos, ``Going deeper than microscopy: the optical imaging frontier
  in biology,'' \emph{Nat. Methods}, vol.~7, no.~8, pp. 603--614, August 2010.

\bibitem{Rudin.etal1992}
L.~I. Rudin, S.~Osher, and E.~Fatemi, ``Nonlinear total variation based noise
  removal algorithms,'' \emph{Physica D}, vol.~60, no. 1--4, pp. 259--268,
  November 1992.

\bibitem{Figueiredo.Nowak2003}
M.~A.~T. Figueiredo and R.~D. Nowak, ``An {EM} algorithm for wavelet-based
  image restoration,'' \emph{IEEE Trans. Image Process.}, vol.~12, no.~8, pp.
  906--916, August 2003.

\bibitem{Daubechies.etal2004}
I.~Daubechies, M.~Defrise, and C.~D. Mol, ``An iterative thresholding algorithm
  for linear inverse problems with a sparsity constraint,'' \emph{Commun. Pure
  Appl. Math.}, vol.~57, no.~11, pp. 1413--1457, November 2004.

\bibitem{Bect.etal2004}
J.~Bect, L.~Blanc-Feraud, G.~Aubert, and A.~Chambolle, ``A $\ell_1$-unified
  variational framework for image restoration,'' in \emph{Proc. {ECCV}},
  Springer, Ed., vol. 3024, New York, 2004, pp. 1--13.

\bibitem{Belkebir.etal2005}
K.~Belkebir, P.~C. Chaumet, and A.~Sentenac, ``Superresolution in total
  internal reflection tomography,'' \emph{J. Opt. Soc. Am. A}, vol.~22, no.~9,
  pp. 1889--1897, September 2005.

\bibitem{Chaumet.Belkebir}
P.~C. Chaumet and K.~Belkebir, ``Three-dimensional reconstruction from real
  data using a conjugate gradient-coupled dipole method,'' \emph{Inv. Probl.},
  vol.~25, no.~2, p. 024003, 2009.

\bibitem{vandenBerg.Kleinman1997}
P.~M. {van den Berg} and R.~E. Kleinman, ``A contrast source inversion
  method,'' \emph{Inv. Probl.}, vol.~13, no.~6, pp. 1607--1620, December 1997.

\bibitem{Abubakar.etal2005}
A.~Abubakar, P.~M. {van den Berg}, and T.~M. Habashy, ``Application of the
  multiplicative regularized contrast source inversion method tm- and
  te-polarized experimental fresnel data,'' \emph{Inv. Probl.}, vol.~21, no.~6,
  pp. S5--S14, 2005.

\bibitem{Bevacquad.etal2017}
M.~T. Bevacquad, L.~Crocco, L.~{Di Donato}, and T.~Isernia, ``Non-linear
  inverse scattering via sparsity regularized contrast source inversion,''
  \emph{IEEE Trans. Comp. Imag.}, 2017.

\bibitem{Belkebir.Sentenac2003}
K.~Belkebir and A.~Sentenac, ``High-resolution optical diffraction
  microscopy,'' \emph{J. Opt. Soc. Am. A}, vol.~20, no.~7, pp. 1223--1229, July
  2003.

\bibitem{Mudry.etal2012}
E.~Mudry, P.~C. Chaumet, K.~Belkebir, and A.~Sentenac, ``Electromagnetic wave
  imaging of three-dimensional targets using a hybrid iterative inversion
  method,'' \emph{Inv. Probl.}, vol.~28, no.~6, p. 065007, April 2012.

\bibitem{Zhang.etal2016}
T.~Zhang, C.~Godavarthi, P.~C. Chaumet, G.~Maire, H.~Giovannini, A.~Talneau,
  M.~Allain, K.~Belkebir, and A.~Sentenac, ``Far-field diffraction microscopy
  at $\lambda/10$ resolution,'' \emph{Optica}, vol.~3, no.~6, pp. 609--612,
  June 2016.

\bibitem{Kamilov.etal2015}
U.~S. Kamilov, I.~N. Papadopoulos, M.~H. Shoreh, A.~Goy, C.~Vonesch, M.~Unser,
  and D.~Psaltis, ``Learning approach to optical tomography,'' \emph{Optica},
  vol.~2, no.~6, pp. 517--522, June 2015.

\bibitem{Kamilov.etal2016}
------, ``Optical tomographic image reconstruction based on beam propagation
  and sparse regularization,'' \emph{IEEE Trans. Comp. Imag.}, vol.~2, no.~1,
  pp. 59--70,, March 2016.

\bibitem{Kamilov.etal2016a}
U.~S. Kamilov, D.~Liu, H.~Mansour, and P.~T. Boufounos, ``A recursive {B}orn
  approach to nonlinear inverse scattering,'' \emph{IEEE Signal Process.
  Lett.}, vol.~23, no.~8, pp. 1052--1056, August 2016.

\bibitem{Liu.etal2017}
H.-Y. Liu, U.~S. Kamilov, D.~Liu, H.~Mansour, and P.~T. Boufounos,
  ``Compressive imaging with iterative forward models,'' in \emph{Proc. {IEEE}
  Int. Conf. Acoustics, Speech and Signal Process. ({ICASSP 2017})}, New
  Orleans, LA, USA, March 5-9, 2017, pp. 6025--6029.

\bibitem{Liu.etal2017b}
H.-Y. Liu, D.~Liu, H.~Mansour, P.~T. Boufounos, L.~Waller, and U.~S. Kamilov,
  ``{SEAGLE}: {S}parsity-driven image reconstruction under multiple
  scattering,'' May 2017, arXiv:1705.04281 [cs.CV].

\bibitem{Soubies.etal2017}
E.~Soubies, T.-A. Pham, and M.~Unser, ``Efficient inversion of
  multiple-scattering model for optical diffraction tomography,'' \emph{Opt.
  Express}, vol.~25, no.~18, pp. 21\,786--21\,800, Sep 2017.

\bibitem{Goodman1996}
J.~W. Goodman, \emph{Introduction to Fourier Optics}, 2nd~ed.\hskip 1em plus
  0.5em minus 0.4em\relax McGraw-Hill, 1996.

\bibitem{Li.Lin2015}
H.~Li and Z.~Lin, ``Accelerated proximal gradient methods for nonconvex
  programming,'' in \emph{Proc. Advances in Neural Information Processing
  Systems 28}, Montreal, Canada, December 7-12 2015.

\bibitem{Ghadimi.Lan2016}
S.~Ghadimi and G.~Lan, ``Accelerated gradient methods for nonconvex nonlinear
  and stochastic programming,'' \emph{Math. Program. Ser. A}, vol. 156, no.~1,
  pp. 59--99, March 2016.

\bibitem{Geffrin.etal2005}
J.-M. Geffrin, P.~Sabouroux, and C.~Eyraud, ``Free space experimental
  scattering database continuation: experimental set-up and measurement
  precision,'' \emph{Inv. Probl.}, vol.~21, no.~6, pp. S117--S130, 2005.

\bibitem{Beck.Teboulle2009a}
A.~Beck and M.~Teboulle, ``Fast gradient-based algorithm for constrained total
  variation image denoising and deblurring problems,'' \emph{IEEE Trans. Image
  Process.}, vol.~18, no.~11, pp. 2419--2434, November 2009.

\bibitem{Kamilov2017}
U.~S. Kamilov, ``A parallel proximal algorithm for anisotropic total variation
  minimization,'' \emph{IEEE Trans. Image Process.}, vol.~26, no.~2, pp.
  539--548, February 2017.

\bibitem{Plessix2006}
R.-E. Plessix, ``A review of the adjoint-state method for computing the
  gradient of a functional with geophysical applications,'' \emph{Geophysical
  Journal International}, vol. 167, no.~2, pp. 495--503, 2006.

\bibitem{Rockafellar2009}
R.~T. Rockafellar and R.~J.-B. Wets, \emph{Variational Analysis}.\hskip 1em
  plus 0.5em minus 0.4em\relax Springer Science \& Business Media, 2009.

\bibitem{Geffrin.Sabouroux2009}
J.-M. Geffrin and P.~Sabouroux, ``Continuing with the {F}resnel database:
  experimental setup and improvements in {3D} scattering measurements,''
  \emph{Inv. Probl.}, vol.~25, no.~2, p. 024001, 2009.

\bibitem{Colton.Kress1992}
D.~Colton and R.~Kress, \emph{Inverse Acoustic and Electromagnetic Scattering
  Theory}.\hskip 1em plus 0.5em minus 0.4em\relax Springer Science \& Business
  Media, 1992, vol.~93.

\end{thebibliography}


\end{document}